%% file: arXiv.tex
\documentclass{article}

\usepackage{fullpage}
\usepackage{amsmath,amsfonts,amssymb,amsthm}
\usepackage{hyperref}
\usepackage[capitalise]{cleveref}
\usepackage{natbib}

\usepackage[utf8]{inputenc}
\usepackage{nicefrac}
\usepackage{xcolor}
\usepackage{mathtools}
\usepackage{xspace}

\usepackage{enumitem}

\newcommand{\remove}[1]{}
\renewcommand{\cref}{\Cref}
\DeclareSymbolFont{AMSb}{U}{msb}{m}{n}
\DeclareMathSymbol{\N}{\mathbin}{AMSb}{"4E}
\DeclareMathSymbol{\Z}{\mathbin}{AMSb}{"5A}
\DeclareMathSymbol{\R}{\mathbin}{AMSb}{"52}
\DeclareMathSymbol{\Q}{\mathbin}{AMSb}{"51}
\DeclareMathSymbol{\erert}{\mathbin}{AMSb}{"50}
\DeclareMathSymbol{\I}{\mathbin}{AMSb}{"49}
\DeclareMathSymbol{\C}{\mathbin}{AMSb}{"43}

\makeatletter
\renewcommand\paragraph{\@startsection{paragraph}{5}{\z@}%
	{3.25ex \@plus1ex \@minus.2ex}%
	{-1em}%
	{\normalfont\normalsize\bfseries}}

\newcommand{\Alg}{\AAA}

\newcommand{\AlgHalfSpace}{\AAA_{\rm LearnHalfSpace}}
\newcommand{\AlgRecConcave}{\AAA_{\rm RecConcave}}
\newcommand{\AlgFindDeepPoint}{\AAA_{\rm FindDeepPoint}}
\newcommand{\AlgOptimizeHighDimFunc}{\AAA_{\rm OptimizeHighDimFunc}}

\newcommand{\AAA}{\mathcal A}

\newcommand{\eps}{\varepsilon}
\newcommand{\z}{\mathrm{z}}

\newcommand{\error}{{\rm error}}
\newcommand{\db}{S}

\newcommand{\size}[1]{\left|#1\right|}

\newcommand{\poly}{\mathop{\rm poly}}

\newcommand{\hs}{\operatorname{\rm hs}}
\newcommand{\hp}{\operatorname{\rm hp}}

\newcommand{\halfspace}{\operatorname{\tt HALFSPACE}}

\newcommand{\set}[1]{\left\{ #1 \right\}}

\def\1{\operatorname*{\mathbb{1}}}

\def\Q{\operatorname*{\mathbb{Q}}}
\def\poly{\mathop{\rm{poly}}\nolimits}

\newcommand{\pt}[1]{{\bf #1}}

\newtheorem{theorem}{Theorem}[section]
\newtheorem*{theorem*}{Theorem}
\newtheorem{lemma}[theorem]{Lemma}
\newtheorem*{lemma*}{Lemma}
\newtheorem{claim}[theorem]{Claim}
\newtheorem{proposition}[theorem]{Proposition}
\newtheorem{observation}[theorem]{Observation}
\newtheorem{requirement}[theorem]{Requirement}

\newtheorem{question}[theorem]{Question}

\theoremstyle{definition}
\newtheorem{definition}[theorem]{Definition}

\newtheorem{example}[theorem]{Example}
\newtheorem{fact}[theorem]{Fact}

\newcommand{\polylog}{{\rm polylog}}
\newcommand{\paren}[1]{\left(#1\right)}
\newcommand{\iseg}[1]{[[{#1}]]}

\newcommand{\wlg} {without loss of generality\xspace}
\newcommand{\px} {\pt{x}}
\newcommand{\py} {\pt{y}}
\newcommand{\pz} {\pt{z}}
\newcommand{\pa} {\pt{a}}

\newcommand{\pA} {\pt{A}}

\newcommand{\eqdef}{:=}
\newcommand{\iprod}[1]{\langle #1 \rangle}

\newcommand{\val}{{\rm val}}
\newcommand{\depth}{{\rm depth}}
\newcommand{\cdepth}{{\rm cdepth}}

\newcommand{\chull}{{\rm ConvexHull}}
\newcommand{\tX}{\tilde{X}}
\newcommand{\tx}{\tilde{x}}
\newcommand{\tY}{\tilde{Y}}
\newcommand{\tcX}{\tilde{\cX}}
\newcommand{\dist}{{\rm dist}}
\newcommand{\Conv}{{\rm Conv}}

\newcommand{\cH}{{\cal{H}}}
\newcommand{\cA}{\mathcal{A}}

\ifdefined\cS
\renewcommand{\cS}{\mathcal{S}}
\else
\newcommand{\cS}{\mathcal{S}}
\fi

\newcommand{\cV}{\mathcal{V}}
\newcommand{\cD}{\mathcal{D}}

\newcommand{\cT}{\mathcal{T}}
\ifdefined\cP
\renewcommand{\cP}{\mathcal{P}}
\else
\newcommand{\cP}{\mathcal{P}}
\fi

\newcommand{\cF}{\mathcal{F}}

\newcommand{\cE}{\mathcal{E}}
\newcommand{\cC}{\mathcal{C}}

\newcommand{\cX}{\mathcal{X}}

\DeclarePairedDelimiter{\norm}{\lVert}{\rVert}

\renewcommand{\Pr}{{\mathrm {Pr}}}
\newcommand{\pr}[1]{\Pr\left[#1\right]}
\newcommand{\ppr}[2]{\Pr_{#1}\left[#2\right]}

\ifdefined\IsDraft
\newcommand{\authnote}[2]{{\bf [{\color{red} #1's Note:} {\color{blue} #2}]}}
\else
\newcommand{\authnote}[2]{}
\fi

\newcommand{\Enote}[1]{\authnote{Eliad}{#1}}

\crefname{proposition}{proposition}{Propositions}

\title{Private Learning of Halfspaces: Simplifying the Construction and Reducing the Sample Complexity}
\author{%
	Haim Kaplan\thanks{Tel Aviv University and Google Research. \texttt{haimk@tau.ac.il}. Partially supported by Israel Science Foundation (grant 1595/19), German-Israeli Foundation (grant 1367/2017), and the Blavatnik Family Foundation.}
	\and
	Yishay Mansour\thanks{Tel Aviv University and Google Research. \texttt{mansour.yishay@gmail.com}. This project has received funding from the European Research Council (ERC) under the European Union’s Horizon 2020 research and innovation program (grant agreement No. 882396), and by the Israel Science Foundation (grant number 993/17).}
	\and
	Uri Stemmer\thanks{Ben-Gurion University and Google Research. \texttt{u@uri.co.il}. Supported in part by the Israel Science Foundation (grant 1871/19), and by the Cyber Security Research Center at Ben-Gurion University of the Negev.}
	\and
	Eliad Tsfadia\thanks{Tel Aviv University and Google Research. \texttt{eliadtsfadia@gmail.com}.}
}

\begin{document}

\maketitle
\begin{abstract}
We present a differentially private learner for halfspaces over a finite grid $G$ in $\R^d$ with sample complexity $\approx d^{2.5}\cdot 2^{\log^*|G|}$, which improves the state-of-the-art result of [Beimel et al., COLT 2019] by a $d^2$ factor. The building block for our learner is a new differentially private algorithm for approximately solving the linear feasibility problem: Given a feasible collection of $m$ linear constraints of the form $Ax\geq b$, the task is to {\em privately} identify a solution $x$ that satisfies {\em most} of the constraints. Our algorithm is iterative, where each iteration determines the next coordinate of the constructed solution $x$.

\end{abstract}

\input{Introduction}
\input{Preliminaries}

\input{OptimizingHighDimFunc}
\input{LearningDeepPoints}
\input{LearningHalfspaces}

\section{Open Questions}\label{sec:openQuestions}
It is still remains open what is the minimal sample complexity that is required for learning halfspaces with an (approximate) differential privacy. 
Our work provides a new upper bound of $\approx d^{2.5}\cdot2^{O(\log^*X)}$ which improves the state-of-the-art result of \cite{BeimelMNS19} by a $d^2$ factor, and improves the generic upper bound of \cite{KLNRS11} whenever (roughly) $d < \log^2 X$.\footnote{We remark that even when $d>\log^2 X$, we offer significant improvements over the generic learner in terms of runtime. In particular, our algorithm runs in time (roughly) $n^d$, where $n$ is the number of samples, while the generic learner has a runtime of at least $X^{d^2}$.} Yet, there is still a gap from the best known lower bound of $\Omega(d \cdot \log^* X)$ for proper learning (\cite{BNSV15}) and $\Omega(d + \log^* X)$ for improper learning. In particular, it is still remains open whether we can avoid the exponential dependency in $\log^* X$ for $d > 1$. One option for answering it is by finding a different $1$-dimensional quasi-concave optimization
that only requires polynomial dependency in $\log^* X$, since RecConcave, the optimization that we are using, requires exponential
dependency. Indeed, a recent work of \cite{KaplanLMNS20} shows
an (almost) linear dependency in $\log^* X$ for $1$-dimensional thresholds, which is a special case of a quasi-concave optimization, and it still remains open whether this result can be extended to the quasi-concave optimization case.

%\bibliographystyle{plain}
%%% Change this to match the name of your BIB file
\bibliographystyle{abbrvnat}
\bibliography{pacparity}
\appendix
\input{RecConcave}
\input{MissingProofs}

\end{document}

%% file: Introduction.tex
\section{Introduction}

Machine learning is an extremely beneficial technology, helping us improve upon nearly all aspects of
life. However, while the benefits of this technology are rather self-evident, it is not without risks. In particular, machine learning models are often trained on sensitive personal information, a fact which may pose serious privacy threats for the training data.
These threats, together with the increasing awareness and demand for user privacy, motivated a long line of work focused on developing {\em private learning algorithms} that provide rigorous privacy guarantees for their training data.

We can think of a private learner as an algorithm that operates on a database containing {\em labeled} individual information, and outputs a hypothesis that predicts the labels of unseen individuals. For example, consider a medical database in which every row contains the medical history of one individual together with a yes/no label indicating whether this individual suffers from some disease. Given this database, a learning algorithm might try to predict whether a new patient suffers from this disease given her medical history. The privacy requirement is that, informally, the output of the learner (the chosen hypothesis) leaks very little information on any particular individual from the database. Formally,

\begin{definition}[\cite{DMNS06}]\label{def:dpIntro}
	Let $\AAA$ be a randomized algorithm that operates on databases.
	Algorithm $\AAA$ is $(\eps,\delta)$-{\em differentially private} if for any two databases $\cS,\cS'$ that differ in one row, and any event $\cT$, we have 
	$\Pr[\AAA(\cS)\in \cT]\leq e^{\eps}\cdot \Pr[\AAA(\cS')\in \cT]+\delta.$ 
	The definition is referred to as {\em pure} differential privacy when $\delta=0$, and {\em approximate} differential privacy when $\delta>0$.
\end{definition}

When constructing private learners, there is a strong tension between the privacy requirement and the utility that can be achieved; one very important and natural measure for this tradeoff is the amount of data required to achieve both goals simultaneously, a.k.a.\ the {\em sample complexity}. This measure is crucial to the practice as it determines the amount of {\em individual data} that must be collected before starting the analysis in the first place.

Recall that the sample complexity of non-private learning is fully characterized by the VC dimension of the hypothesis class. For {\em pure}-private learners (i.e., learners that satisfy pure-differential privacy), there is an analogous characterizations in terms of a measure called {\em the representation dimension} \citep{BNS13}. %However, for many interesting cases, the representation dimension is significantly higher than the VC dimension, and hence, learning with pure-differential privacy often requires significantly more data than non-private learning. For example, 
However, the situation is far less understood for {\em approximate} private learning, and there is currently no tight characterization
for the sample complexity of {\em approximate} private learners.\footnote{We remark that there is a loose characterization for private learning in terms of the {\em Littlestone dimension}~\cite{ALMM18,bun2020equivalence}. Specifically, these results state that the sample complexity of privately learning a class $C$ is somewhere between $\Omega(\log^*L)$ and $2^{O(L)}$, where $L$ is the Littlestone dimension of $C$. In our contaxt, for learning halfspaces, these results do not provide meaningful bounds on the sample complexity.}

In this work we investigate the sample complexity of private learning for one of the most basic and important learning tasks -- learning halfspaces. We begin by surveying the existing results.

\subsection{Existing Results}
Recall that the VC dimension of the class of all halfspaces over $\R^d$ is $d$, and hence a sample of size $O(d)$ suffices to learn halfspaces non-privately (we omit throughout the introduction the dependency of the sample complexity in the accuracy, confidence, and privacy parameters). 
In contrast, it turns out that with differential privacy,
learning halfspaces over $\R^d$ is {\em impossible}, even with approximate differential privacy, and even when $d=1$ \citep{FX14,BNSV15,ALMM18}. 

In more details,
let $X\in\N$ be a discretization parameter, let $\cX=\{x\in\Z : |x|\leq X\}$, and consider the task of learning halfspaces over the {\em finite} grid $\cX^d\subseteq\R^d$. In other words, consider the task of learning halfspaces under the promise that the underlying distribution is supported on (a subset of) the finite grid $\cX^d$. 
For pure-private learning, \cite{FX14} showed a lower bound of $\Omega\left(d^2\cdot\log X\right)$ on the sample complexity of this task. 
This lower bound is tight, as a pure-private learner with sample complexity $\Theta\left(d^2\cdot\log X\right)$ can be obtained using the generic upper bound of \cite{KLNRS11}.
This should be contrasted with the non-private sample complexity, which is linear in $d$ and independent of $X$.

For the case of $d=1$, \cite{BNS13b} showed that the lower bound of \cite{FX14} can be circumvented by relaxing the privacy guarantees from pure to approximate differential privacy. Specifically, they presented an approximate-private learner for 1-dimensional halfspaces with sample complexity $2^{O(\log^*X)}$. The building block in their construction is a differentially private algorithm, called $\AlgRecConcave$, for approximately optimizing {\em quasi-concave} functions.\footnote{A function $Q$ is {\em quasi-concave} if for any $x' \leq x \leq x''$ it holds that $Q(x) \geq \min\set{Q(x'),Q(x'')}$.}
%\Unote{should we add more details here on what is a quasi-concave function and what are the guarantees of $\AlgRecConcave$?} 

Following the work of \cite{BNS13b}, two additional algorithms for privately learning 1-dimensional halfspaces with sample complexity $2^{O(\log^*X)}$ were given by \cite{BNSV15} and by \cite{BunDRS18}. Recently, an algorithm with sample complexity $\tilde{O}\left(\left(\log^*X\right)^{1.5}\right)$ was given by \cite{KaplanLMNS20} (again for $d=1$).
In light of these positive results, it might be tempting to guess that the sample complexity of privately learning halfspaces can be made independent of the discretization parameter $X$. However, as \cite{BNSV15} and \cite{ALMM18} showed, this is not the case, and every approximate-private learner for 1-dimensional halfspaces over $\cX$ must have sample complexity at least $\Omega(\log^*X)$.
Observe that, in particular, this means that learning halfspaces over $\R$ is impossible with differential privacy (even for $d=1$).

Recently, \cite{BeimelMNS19} presented an approximate-private learner for $d$-dimensional halfspaces (over $\cX^d$) with sample complexity $\approx d^{4.5}\cdot2^{O(\log^*X)}$.
Their algorithm is based on a reduction to the task of privately finding a point in the convex hull of a given input dataset. Specifically, given a dataset $\cS$ containing points from the finite grid $\cX^d\subseteq\R^d$, consider the task of (privately) finding a point $y\in\R^d$ that belongs to the convex hull of the points in $\cS$. \cite{BeimelMNS19} presented an iterative algorithm for this task that is based on the following paradigm: Suppose that we have identified values for the first $i-1$ coordinates $x^*_1,\dots,x^*_{i-1}$ for which we know that there exists a completion $\tx_i,\dots,\tx_d$ such that $(x^*_1,\dots,x^*_{i-1},\tx_i,\dots,\tx_d)$ belongs to the convex hull of the input points. Then, during the $i$th iteration of the algorithm, we aim to find the next coordinate $x^*_i$ such that $(x^*_1,\dots,x^*_{i})$ can be completed to a point in the convex hull. To that end, \cite{BeimelMNS19} formulated the task of identifying the next coordinate $x^*_{i}$ as a (1-dimensional) quasi-concave optimization problem, and used algorithm $\AlgRecConcave$ of \cite{BNS13b} for privately solving it. This strategy is useful because algorithm $\AlgRecConcave$ is very efficient (in terms of sample complexity) in optimizing 1-dimensional quasi-concave functions (requires only $\approx2^{O(\log^*X)}$ many samples). 
This paradigm (together with a reduction from privately learning halfspaces to privately finding a point in the convex hull) resulted in a private learner for halfspaces over $\cX^d$ with sample complexity $\approx d^{4.5}\cdot2^{O(\log^*X)}$.
%This resulted in a private algorithm that succeeds in identifying a point in the convex hull the input dataset $S\subseteq \cX^d$, provided that $|S|\gtrsim d^{2.5}\cdot2^{O(\log^*X)}$. 

\subsection{Our Results}

In this work, we generalize the technique that \cite{BeimelMNS19} applied to the problem of finding a point in the convex hull, which we refer to as the ``RecConcave paradigm'', and reformulate it as a general method for privately optimizing high dimensional functions. 
%That is, we show that, with the right formulation, the RecConcave paradigm can be applied much more broadly, and is not limited to the problem of finding a point in the convex hull. 
As a result, we obtain a private PAC learner for halfspaces with an improved sample complexity of $\approx d^{2.5}\cdot2^{O(\log^*X)}$.

\begin{theorem}[Learning Halfspaces, Informal]\label{thm:introHalfspaces}
	Let $ \alpha,\beta,\eps \leq 1$ and $\delta < 1/2$ and let $\cX \subset \R$. There exists an $(\eps,\delta)$-differentially private $(\alpha,\beta)$-PAC learner for halfspaces over examples from $\cX^d$ with sample complexity
	$s = d^{2.5}\cdot 2^{O(\log^* X)}\cdot \frac1{\eps \alpha}\cdot \polylog\paren{\frac{d}{\alpha \beta \eps \delta}}$.
\end{theorem}

To obtain Theorem~\ref{thm:introHalfspaces}, we show that 
the task of privately learning halfspaces reduces to the task of privately solving the linear feasibility problem (as defined below) with essentially the same parameters, and solve the linear feasibility problem using our generalized RecConcave paradigm.

\vspace{-10pt}
\paragraph{The Linear Feasibility Problem.} Let $\cX=\{x\in\Z : |x|\leq X\}$ for some parameter $X\in\N$. 
In the linear feasibility problem, we are given a feasible collection of $m$ linear constraints over $d$ variables $x_1,\dots,x_d$, and the goal is to find a solution in $\R^d$ that satisfies all constraints. Each constraint has the form $\sum_{i=1}^d a_i x_i \geq b$ for some $a_1,\ldots,a_d,b \in \cX$.

%\medskip
Without privacy considerations, this well-known problem can be solved, e.g., using the Ellipsoid Method or the Interior Point Method. 
In the private version of this problem, we would like to come up with a solution to the system in a way that is insensitive to any (arbitrary) change of single constraint (in the sense of differential privacy, see Definition~\ref{def:dpIntro}).
It is easy to see that with differential privacy, one cannot hope for an exact solution to this problem (i.e., a solution that satisfies {\em all} constraints). This is because changing a single constraint, which has basically no effect on the outcome of a private algorithm, may completely change the feasibility area. Therefore, in the private version of this problem we only aim to satisfy {\em most} of the constraints. Specifically, we say that an algorithm {\em $(\alpha,\beta)$-solves} the $(X,d,m)$-linear feasibility problem, if for every feasible collection of $m$ linear constraints over $d$ variables with coefficients from $\cX$, with probability $1-\beta$ the algorithm finds a solution $\px = (x_1,\ldots,x_d)$ that satisfies at least $(1-\alpha)m$ constraints. 

\begin{question}
	What is the minimal number of constraints $m$, as a function of $X,d,\alpha,\beta,\eps,\delta$, for which there exists an $(\eps,\delta)$-differentially private algorithm that $(\alpha,\beta)$-solves the $(X,d,m)$-linear feasibility problem?
\end{question}

%Even with this relaxation, it can be shown that without restricting the set of all possible constraints to be finite, it is impossible to privately solve this problem. For $d=1$, the problem is equivalent to learning a one-dimensional threshold, where \Enote{add ref} showed that in order to privately learn a threshold we need to have at least $m = O(\log^* X)$ samples, letting $X$ be the size of the domain of all possible thresholds. Hence, in the private setting of the linear feasibility problem, we must have a finite grid $\cX$ such that the set of all possible constraints are ones with coefficients $a_1,\ldots,a_d,b \in \cX$.

Observe that this question is trivial without the privacy requirement (it can be solved easily when $m=1$). 
However, the picture is quite different with differential privacy. In particular, 
all the lower bounds we mentioned before on the sample complexity of learning halfspaces yield lower bounds on the number of constraints $m$ needed to privately solve the linear feasibility problem. 
We prove the following theorem.

\begin{theorem}[Linear Feasibility Problem, Informal]\label{thm:LP}
	Let $\alpha,\beta,\eps \leq 1$ and $\delta < 1/2$ and let $X\in\N$. There exists an $(\eps,\delta)$-differentially private algorithm that 
	$(\alpha,\beta)$-solves the $(X,d,m)$-linear feasibility problem, for every
	$m\geq d^{2.5}\cdot 2^{O(\log^* X)}\cdot \frac1{\eps \alpha}\cdot \polylog\paren{\frac{d}{\beta \delta}}$.
\end{theorem}

\vspace{-10pt}
\paragraph{A Generalized RecConcave Paradigm.}
Let $f(\cS,\px)$ be a low-sensitivity function that takes a database $\cS$ and a high dimensional point $\px$, and returns a real number which is identified as the ``score'' of the point $\px$ w.r.t.\ the database $\cS$.\footnote{The {\em sensitivity} of the function $f$ is the maximal difference by which the value of $f(\cS,\px)$ can change when modifying one element of the database $\cS$. See Section~\ref{sec:prelims} for a formal definition.} Now suppose that, given an input database $\cS$, we would like to (privately) identify a point $\px$ such that $f(\cS,\px)$ is approximately maximized. 

\begin{example}
	{\em To solve the linear feasibility problem we can define the function $f(\cS,\px)$ as the number of constraints in $\cS$ that are satisfied by $\px$, a quantity which we denote by $\depth_{\cS}(x_1,\ldots,x_d)$. Note that an approximate maximizer for this $f$ is a good solution to the linear feasibility problem, i.e., it satisfies most of the constraints.}
\end{example}

A naive attempt for using the RecConcave paradigm in order to privately maximize $f$ is to define the following function $Q$ (for every 
$i\in[d]$ and every fixing of $x^*_1,\ldots,x^*_{i-1}$).
\begin{align*}
	Q_{x^*_1,\ldots,x^*_{i-1}}(x_i)=\max_{\tx_{i+1},\dots,\tx_d}\{f\left(\cS, x^*_1,\ldots,x^*_{i-1}, x_i, \tx_{i+1},\dots,\tx_d \right)\}.
\end{align*}
Now, if it happens that $Q$ is quasi-concave, then one can apply $\AlgRecConcave$ coordinate by coordinate in order to privately find a solution $\px$ that approximately maximizes $f(\cS,\px)$. To see this, suppose that we find (using $\AlgRecConcave$) a value $x^*_1$ for the first coordinate that approximately maximizes $Q(\cdot)$. By the definition of $Q$, this guarantees that there exists a completion $(\tx_2,\dots,\tx_d)$ such that $f(\cS,x^*_1,\tx_2,\dots,\tx_d)$ is almost as high as $\max_{\px}\{f(\cS,\px)\}$. Hence, by committing to $x^*_1$ we do not lose much in terms of the maximum attainable value of $f$. Similarly, in every iteration we identify a value for the next coordinate without losing too much in the maximum attainable value of $f$.

The problem is that, in general, the above function $Q$ is not necessarily quasi-concave. In particular, in the linear feasibility problem where $f(\cS,\px) = \depth_{\cS}(\px)$ (i.e., the number of constraints in $\cS$ that are satisfied by $\px$), the resulting function $Q$ is not quasi-concave.\footnote{For instance, consider the 2-dimensional constraints $x_2 \geq x_1$ and $x_2 \leq -x_1$. Then under the fixing $x^*_1 = 1$, the depth of $x_2 = 0$ is $0$ while the depth of $x_2 \in \set{-1,1}$ is $1$, yielding that $Q_{x^*_1}(0) < \min\set{Q_{x^*_1}(-1),Q_{x^*_1}(1)}$, and so $Q_{x^*_1}$ is not quasi-concave.}

In order to overcome this issue, we present the following technique which we refer to as the generalized RecConcave Paradigm.\footnote{We remark that the presentation here is oversimplified, and hides many of the challenges that arise in the actual analysis.}
We define the ``convexification'' of a function $f$ to be the function $f_{\Conv}(\cS,\px)$ that outputs the maximal $y \in \R$ for which the point $\px$ is a convex combination of points $\pz \in \R^d$ with $f(\cS,\pz) \geq y$. 
In other words, for any $y \in \R$, we consider the set $\cD_{\cS}(y)=\left\{ \pz\in\R^d : f(\cS,\pz) \geq y\right\}$, and denote $\cC_{\cS}(y)=\chull(\cD_{\cS}(y))$. Then, $f_{\Conv}(\cS,\px) \eqdef \max\{y : \px\in \cC_{\cS}(y)\}$. 
We show that with this function $f_{\Conv}(\cS,\px)$, the resulting function 
\begin{align*}
	Q_{x^*_1,\ldots,x^*_{i-1}}(x_i)=\max_{\tx_{i+1},\dots,\tx_d}\{f_{\Conv}\left(\cS, x^*_1,\ldots,x^*_{i-1}, x_i, \tx_{i+1},\dots,\tx_d \right)\}
\end{align*}
is indeed quasi-concave for any fixing of $x^*_1,\ldots,x^*_{i-1}$ (no matter how the function $f$ is defined). The function $f_{\Conv}$ can, therefore, be approximately maximized (privately) coordinate by coordinate using $\AlgRecConcave$.
Furthermore, if $f$ has the property that points $\px$ with high $f_{\Conv}(\cS,\px)$ also have (somewhat) high $f(\cS,\px)$, then $f$ can be privately maximized (approximately) by maximizing the function $f_{\Conv}$. 
Going back to the linear feasibility problem, we denote by $\cdepth_{\cS}(\px) = f_{\Conv}(\cS,\px)$ the convexification of the function $f(\cS,\px) = \depth_{\cS}(\px)$. We then show that every point that has $\cdepth = (1-\lambda)\size{\cS}$ must have $\depth \geq (1-(d+1)\lambda)\size{\cS}$. Applying the aforementioned method on the function $\depth$ results in a differentially private algorithm for solving the $(X,d,m)$-linear feasibility problem whenever $m\gtrsim d^{2.5}\cdot2^{O(\log^*X)}$.

\subsection{Other Related Work}\label{sec:otherRelated}
\cite{DV08} showed an efficient (non-private) learner for the linear feasibility problem that works in (a variant of) the {\em statistical query (SQ)} model of \cite{Kearns98}. It is known that algorithms operating in the SQ model can be transformed to preserve differential privacy \citep{BDMN05}, and the algorithm of \cite{DV08} yields a differentially private efficient algorithm for solving the $(X,d,m)$-linear feasibility problem for $m\geq\poly(d,\log|X|)$. 
Another related work is that of \cite{HsuRRU14} who studied a variant of the linear feasibility problem with a certain large-margin assumption. Specifically, given a feasible collection of linear constraints of the form $\sum_{i=1}^d a_i x_i \geq 0$, their algorithm finds a solution $\px^*$ that approximately satisfies most of them (that is, $\sum_{i=1}^d a_i x^*_i \geq - c$ for most of the constraints, for a ``margin parameter'' $c > 0$). 
Large-margin assumptions were also utilized by \cite{BDMN05} and \cite{NguyenUlZa19} who designed efficient private learners for learning large-margin halfspaces. In addition, several other works developed tools that implicitly imply private learning of large-margin halfspaces, such as the works of \cite{ChaudhuriMS11} and \cite{BST14}. We remark that in this work we do not make large-margin assumptions. 

%% file: Preliminaries.tex
\section{Preliminaries}\label{sec:prelims}

In this section we state basic preliminaries from learning theory and differential privacy, introduce a tool that enables our constructions, describe the geometric objects we use throughout the paper, and present some of their properties.

\paragraph{Notations.} 
We use calligraphic letters to denote sets and boldface for vectors and matrices.
We let $\N_0 = \N \cup \set{0}$.
For $\px = (x_1,\ldots,x_d) \in \R^d$ and $\py = (y_1,\ldots,y_d) \in \R^d$, we let $\iprod{\px,\py} \eqdef \sum_{i=1}^d x_i y_i$ be the inner-product of $\px$ and $\py$, and $\norm{\px} \eqdef \sqrt{\iprod{\px,\px}}$ be the norm of $\px$.
For two integers $a \leq b$, let $\iseg{a,b} \eqdef \set{a, a+1, \ldots, b}$ and let $\iseg{\pm a} \eqdef \iseg{-\size{a},\size{a}}$.
Given sets $\cS_1,\ldots,\cS_k$ and $k$-input function $f$, let  $f(\cS_1,\ldots,\cS_k) \eqdef \set{f(x_1,\ldots,x_j) \colon x_i\in \cS_i}$, e.g., $\iseg{\pm 5} / \iseg{7,20} = \set{x/y \colon x \in \iseg{\pm 5}, y \in \iseg{7,20}}$. Given a set $\cX$ we let $\cX^*$ be the set of all possible multisets whose elements are taken (possibly with repetitions) from the set $\cX$.

\remove{
The input of our learning halfspaces algorithm is a multiset $\cS \in (\cX^d \times \set{-1,1})^*$ where $\cX = \iseg{\pm X}$ for some integer $X$.
Databases $\cS_1$ and $\cS_2$ are called {\em neighboring} if they differ in exactly one entry. Throughout this paper we use $\eps$ and $\delta$ for the privacy parameters, $\alpha$ for the error parameter, and $\beta$ for the confidence parameter.}

\subsection{Preliminaries from Differential Privacy} 

Consider a database where each record contains information of an individual. An algorithm is said to preserve differential privacy if a change of a single record of the database (i.e., information of an individual) does not significantly change the output distribution of the algorithm. Intuitively, this means that the information inferred about an individual from the output of a differentially-private algorithm is similar to the information that would be inferred had the individual's record been arbitrarily modified or removed. Formally:

\begin{definition}[Differential privacy~\citep{DMNS06,DKMMN06}] \label{def:dp} 
	A randomized algorithm $\Alg$ is $(\eps,\delta)$-differentially private if for all neighboring databases $\db_1,\db_2$ (i.e., differ by exactly one entry), and for all sets $\mathcal{F}$ of outputs,
	\begin{eqnarray}
	\label{eqn:diffPrivDef}
	& \Pr[\Alg(\db_1) \in \mathcal{F}] \leq \exp(\eps) \cdot \Pr[\Alg(\db_2) \in \mathcal{F}] + \delta,  &
	\end{eqnarray}
	where the probability is taken over the random coins of $\Alg$. 
	When $\delta=0$ we omit it and say that $\Alg$ preserves $\eps$-differential privacy.
\end{definition}
We use the term {\em pure} differential privacy when $\delta=0$ and the term {\em approximate} differential privacy when $\delta>0$, in which case $\delta$ is typically a negligible function of the database size $m$.

We will later present algorithms that access their input database using (several) differentially private algorithms. We will use the following composition theorems. 

\begin{theorem}[Basic composition]\label{thm:composition1}
	If $\Alg_1$ and $\Alg_2$ satisfy $(\eps_1,\delta_1)$ and $(\eps_2,\delta_2)$ differential privacy, respectively, then their concatenation $\Alg(\cS)=\langle \Alg_1(\cS),\Alg_2(\cS) \rangle$ satisfies $(\eps_1+\eps_2,\delta_1+\delta_2)$-differential privacy.
\end{theorem}

Moreover, a similar theorem holds for the adaptive case, where an algorithm  uses  $k$ {\em adaptively chosen} differentially private algorithms (that is, when the choice of the next differentially private algorithm that is used depends on the outputs of the previous differentially private algorithms).

\begin{theorem}[\citep{DKMMN06, DworkLei}]\label{thm:composition3}
	An algorithm that adaptively uses $k$  algorithms that preserves $(\eps/k,\delta/k)$-differential privacy (and does not access the database otherwise) ensures $(\eps,\delta)$-differential privacy.
\end{theorem}

Note that the privacy guaranties of the above bound deteriorates linearly with the number of interactions. By bounding the {\em expected} privacy loss in each interaction (as opposed to worst-case), \cite{DRV10} showed the following stronger composition theorem, where privacy deteriorates (roughly) as $\sqrt{k}\eps+k\eps^2$ (rather than $k\eps$).

\begin{theorem}[Advanced composition~\cite{DRV10}, restated]\label{thm:composition2}
	Let $0<\eps_0,\delta'\leq1$, and let $\delta_0\in[0,1]$. An algorithm that adaptively uses $k$ algorithms that preserves $(\eps_0,\delta_0)$-differential privacy (and does not access the database otherwise) ensures $(\eps,\delta)$-differential privacy, where $\eps=\sqrt{2k\ln(1/\delta')}\cdot\eps_0+2k\eps_0^2$ and $\delta = k\delta_0+\delta'$.
\end{theorem}

\subsection{Preliminaries from Learning Theory}

We next define the probably approximately correct (PAC) model of~\cite{Valiant84}.
A concept $c:\cX\rightarrow \{0,1\}$ is a predicate that labels {\em examples} taken from the domain $\cX$ by either 0 or 1.  A \emph{concept class} $\cC$ over $\cX$ is a set of concepts (predicates) mapping $\cX$ to $\{0,1\}$. A learning algorithm is given examples sampled according to an unknown probability distribution $\mu$ over $\cX$, and labeled according to an unknown {\em target} concept $c\in \cC$. The learning algorithm is successful when it outputs a hypothesis $h$ that approximates the target concept over samples from $\mu$. More formally:

\begin{definition}
	The {\em generalization error} of a hypothesis $h:X\rightarrow\{0,1\}$ is defined as 
	$$\error_{\mu}(c,h)=\Pr_{x \sim \mu}[h(x)\neq c(x)].$$ 
	If $\error_{\mu}(c,h)\leq\alpha$ we say that $h$ is {\em $\alpha$-good} for $c$ and $\mu$.
\end{definition}

\begin{definition}[PAC Learning~\citep{Valiant84}]\label{def:PAC}
	Algorithm $\Alg$ is an {\em $(\alpha,\beta,m)$-PAC learner} for a concept
	class $\cC$ over $\cX$ using hypothesis class $\cH$ if for all 
	concepts $c \in \cC$, all distributions $\mu$ on $\cX$,
	given an input of $m$ samples $\db =(z_1,\ldots,z_m)$, where $z_i=(x_i,c(x_i))$ and each $x_i$
	is drawn i.i.d.\ from $\mu$, algorithm $\Alg$ outputs a
	hypothesis $h\in \cH$ satisfying
	$$\Pr[\error_{\mu}(c,h)  \leq \alpha] \geq 1-\beta,$$
	where the probability is taken over the random choice of
	the examples in $\db$ according to $\mu$ and the random coins  of the learner $\Alg$.
	If $\cH\subseteq \cC$ then $\Alg$ is called a {\em proper} PAC learner; otherwise, it is called an {\em improper} PAC learner.
\end{definition}

\begin{definition}
	For a labeled sample $\db=(x_i,y_i)_{i=1}^m$, the {\em empirical error} of $h$ is
	$$\error_S(h) = \frac{1}{m} |\{i : h(x_i) \neq y_i\}|.$$
\end{definition}

We use the following fact.

\begin{theorem}[\cite{BlumerEHW89}]\label{thm:empirical-to-PAC}
	Let $\cC$ and $\mu$ be a concept class and a distribution over a domain $\cX$. Let $\alpha,\beta > 0$, and $m \geq \frac{48}{\alpha}\paren{10VC(\cC)\log\paren{\frac{48e}{\alpha}} + \log\paren{\frac{5}{\beta}}}$. Suppose that we draw a sample $\cS = (x_i)_{i=1}^m$, where each $x_i$ is drawn i.i.d. from $\mu$. Then
	\begin{align*}
	\pr{\exists c,h \in \cC\text{ s.t. }\error_{\mu}(c,h)\geq \alpha\text{ and }\error_{\cS}(c,h)\leq \alpha/10} \leq \beta.
	\end{align*}
\end{theorem}

\subsection{Private Learning}\label{sec:PPAC}
Consider a learning algorithm $\AAA$ in the probably approximately correct (PAC) model of~\cite{Valiant84}. We say that $\AAA$ is a {\em private} learner if it also satisfies differential privacy w.r.t.\ its training data. Formally,
\begin{definition}[Private PAC Learning~\citep{KLNRS11}]
	%\label{def:private-general}
	Let $\Alg$ be an algorithm that gets an input $\db =(z_1,\ldots,z_m)$. Algorithm $\Alg$ is an {\em $(\eps,\delta)$-differentially private $(\alpha,\beta)$-PAC learner with sample complexity $m$} for a concept
	class $\cC$ over $\cX$ using hypothesis class $\cH$ if
	\begin{description}
		\item{\sc Privacy.} Algorithm $\Alg$ is $(\eps,\delta)$-differentially private (as in Definition \ref{def:dpIntro});
		\item{\sc Utility.} Algorithm $\Alg$ is an {\em $(\alpha,\beta)$-PAC learner} for $\cC$ with sample complexity $m$ using hypothesis class $\cH$.
	\end{description}
	%When $\delta=0$ (pure privacy) we omit it from the list of parameters.
\end{definition}

Note that the utility requirement in the above definition is an average-case requirement, as the learner is only required to do well on typical samples (i.e., samples drawn i.i.d. from a distribution $\mu$ and correctly labeled by a target concept $c\in \cC$). In contrast, the privacy requirement is a worst-case requirement, that must hold for every pair of neighboring databases (no matter how they were generated, even if they are not consistent with any concept in $\cC$).

\subsection{A Private Algorithm for Optimizing Quasi-concave Functions -- $\AlgRecConcave$}
We describe the properties of  algorithm $\AlgRecConcave$ of \cite{BNS16a}. This algorithm is given a quasi-concave function $Q$ (defined below) and a database $\cS$ and privately finds a point $x$ such that $Q(\cS,x)$ is close to its maximum provided that the maximum of $Q(\cS,\cdot)$ is large enough (see~(\ref{eq:largeQ})).% This algorithm is too long to describe it here, and we only describe its properties.

\begin{definition}\label{def:quasiConcave} 
	A function $f$ is quasi-concave if $f(\ell) \geq \min\set{f(i),f(j)}$ for every $i < \ell  < j$.
\end{definition}

\begin{definition}[Sensitivity]\label{def:sensitivity} 
	The sensitivity of a function $f : \cX^* \rightarrow \R$ is the smallest $k$ such that for
	every neighboring databases $\cS,\cS' \in \cX^*$ (i.e., differ in exactly one entry), we have $|f(\cS)-f(\cS')
	| \leq k$. A function $g \colon \cX^* \times \tcX  \rightarrow \R$ is called a sensitivity-$k$ function if 
	for every $x \in \tcX$, the function
	$g(\cdot,x)$ has sensitivity $\leq k$.
\end{definition}

\begin{proposition}[Properties of Algorithm $\AlgRecConcave$~\citep{BNS13b}]\label{prop:aRecConcave}
	Let $Q:\cX^*\times \tcX \rightarrow\R$ be a sensitivity-$1$ function.
	Denote $\tX=\size{\tcX}$ and let $\alpha\leq\frac{1}{2}$ and  $\beta,\eps,\delta,r$ be parameters.
	There exists an $(\eps,\delta)$-differentially private algorithm, called $\AlgRecConcave$, such that the following holds.
	If $\AlgRecConcave$ is executed on a database $\cS \in \cX^*$ such that $Q(\cS,\cdot)$ is quasi-concave and 
	\begin{align}\label{eq:largeQ}
	\max_{i\in \tcX}\{Q(\cS,i)\} \geq r \geq
	8^{\log^* \tX} \cdot \frac{12 \log^* \tX}{\alpha\eps}\log\Big(\frac{192(\log^* \tX)^2}{\beta\delta}\Big).
	\end{align}
	then with probability $1-\beta$ the algorithm outputs an index $j$ s.t.\ 
	$Q(\cS,j)\geq(1-\alpha)r$.
\end{proposition}

We next give a short summary of how Algorithm $\AlgRecConcave$ works.

\cite{BNS13b} observed that a quasi-concave promise problem can be privately approximated using a solution to a smaller instance of a quasi-concave promise problem. Specifically, they showed that for any quasi-concave  function $Q:\cX^*\times\tcX \rightarrow\R$ with a (large enough) promise $r$, there exists a quasi-concave function $Q' :\cX^*\times \tcX' \rightarrow\R$ with a promise $r' = \Omega(\alpha r)$ and with $\size{\tcX'} \approx \log \size{\tcX}$, such that the task of privately finding $j \in\tcX$ with $Q(\cS,j)\geq(1-\alpha)r$ is reduced to the task of privately finding $k \in \tcX'$ with $Q'(\cS,k) \geq (1-\alpha) r'$. This resulted in a recursive algorithm $\AlgRecConcave$ for optimizing $Q$. 
For the sake of completeness, we give more details in \cref{sec:moreRecConcave}.

\remove{
	We next present an alternative definition of quasi-concave functions that is more convenient to use.
	\begin{observation}
		Let $i_{\rm max}$ be a value such that $Q(i_{\rm max})$ is maximal.
		A function $Q(Â·)$ is quasi-concave if and only if $Q(i) \leq Q(j)$ for all $i,j\in \tilde X$ such that $i < j < i_{\rm max}$ or $i_{\rm max} < j < i$.
		
		\begin{enumerate}
			\item
			The function is non-decreasing  before $i_{\rm max}$, that is 
			$Q(i) \leq Q(j)$ for every $i,j\in \tilde{X}$ such that $i < j < i_{\rm max}$, {\em and}
			\item
			The function is non-increasing after $i_{\rm max}$, that is 
			$Q(i) \geq Q(j)$ for every $i,j\in \tilde{X}$ such that $i_{\rm max} < i < j$.
		\end{enumerate}
	\end{observation} 
}

\subsection{Halfspaces and Convex Hull}\label{def:Tukey}

We next define the geometric objects we use in this paper.

\begin{definition}[Halfspaces and Hyperplanes]
	For $\pt{a} = (a_1,\ldots,a_d) \in \R^d \setminus \set{(0,\ldots,0)}$ and $w\in \R$, let the halfspace defined by $(\pt{a},w)$ be
	$\hs_{\pt{a},w} \eqdef \set{\px \in \R^d \colon \iprod{\pa,\px} \geq w}$. For a domain $\cD \subseteq \R^d$  define the concept class $\halfspace(\cD) = \set{c_{\pa,w} \colon \cD \mapsto \set{-1,1}}$, letting $c_{\pa,w}$ be the function that on input $\px \in \cD$ outputs $1$ iff $\px \in \hs_{\pa,w}$.
	The hyperplane $\hp_{\pt{a},w}$ defined by $(\pt{a},w)$ is the set of all points $\pt{x} \in \R^d$ such that $\iprod{\pa,\px}= w$.
	\remove{
		For $\pt{a} = (a_1,\ldots,a_d) \in \R^d \setminus \set{(0,\ldots,0)}$ and $w\in \R$, let the halfspace defined by $(\pt{a},w)$ be
		$\hs_{\pt{a},w} \eqdef \set{\px \in \R^d \colon \iprod{\pa,\px} \geq w}$, and let $\hs_{\pa} \eqdef \hs_{\pt{a},1}$. For a domain $\cD \subseteq \R^d$  define the concept class $\halfspace(\cD) = \set{c_{\pa,w} \colon \cD \mapsto \set{-1,1}}$, letting $c_{\pa,w}$ be the function that on input $\px \in \cD$ outputs $1$ iff $\px \in \hs_{\pa,w}$.
		The hyperplane $\hp_{\pt{a},w}$ defined by $(\pt{a},w)$ is the set of all points $\pt{x} \in \R^d$ such that $\iprod{\pa,\px}= w$, and let $\hp_{\pa} \eqdef \hp_{\pa,1}$.
		For $\pt{a} \in \R^d$ and $w \in \R$, we define $\hs_{\pt{a},w}^1 \eqdef \hs_{\pt{a},w}$ and $\hs_{\pt{a},w}^{-1} \eqdef  \set{\px \in \R^d \colon \iprod{\pa,\px} \leq w}$, and we omit the $w$ in case $w=1$.
	}
	
	\remove{\Enote{Added:}In addition, throughout this paper we let $\hs_{\pt{a}} = \hs_{\pt{a},1}$ and $\hp_{\pt{a}} = \hp_{\pt{a},1}$, and note that for any $w \neq 0$ is holds that \Enote{Not true!} $\hs_{a_1,\dots,a_d,w} = \hs_{a_1/w,\dots,a_d/w}$ and $\hp_{a_1,\dots,a_d,w} = \hp_{a_1/w,\dots,a_d/w}$.}
\end{definition}

\begin{definition}[Convex Hull]
	Let $\cP\subseteq \R^d$ be a set of points. The convex hull of $\cP$, denote by $\chull(\cP)$, is the set of all points $\px \in \R^d$ that are convex combination of elements of $\cP$. That is, $\px \in \chull(\cP)$ iff there exists a finite subset $\cP' \subseteq \cP$ and
	numbers $\set{\lambda_{\pt{y}}}_{\pt{y} \in \cP'}$ such that
	$\sum_{\pt{y} \in \cP'} \lambda_\pt{y} =1$ and
	$\sum_{\pt{y} \in \cP'} \lambda_\pt{y} \pt{y}=\pt{x}$. 
\end{definition}

We use the following fact.

\begin{fact}[Caratheodory's theorem]\label{fact:Caratheodory}
	Let $\cP\subseteq \R^d$ be a set of points. Then any $\px \in \chull(\cP)$ is a convex combination of at most $d+1$ points in $\cP$.
\end{fact}

%% file: OptimizingHighDimFunc.tex
\section{Optimizing High-Dimensional Functions}\label{sec:OptHighDimFunc}

In this section we present our general method for privately optimizing high dimensional functions. In the following, let $\cX$ be a domain and let $f \colon \cX^* \times \R^d \rightarrow \R$ be a function that given a dataset $\cS \in \cX^*$, we would like to approximately maximize $f(\cS,\cdot)$. Formally, given $\alpha,\beta,\epsilon,\delta \in (0,1)$, our goal is to design an $(\eps,\delta)$-differential private algorithm that with probability $1-\beta$ finds $\px^* \in \R^d$ with $f(\cS,\px^*) \geq (1-\alpha) M_{\cS}$ for $M_{\cS} \eqdef \max_{\px} f(\cS,\px)$. We do so by optimizing a different (but related) function $f_{\Conv}$, which we call the ``convexification'' of $f$.

\begin{definition}[The convexification of $f$]\label{def:convexification}
	For $\cS \in \cX^*$ and $y \in \R$, let $\cD_{\cS}(y) \eqdef \set{\pz \in \R^d \colon f(\cS,\pz) \geq y}$ and $\cC_{\cS}(y) \eqdef \chull\paren{\cD_{\cS}(y)}$. We define the convexification of $f$ as the function $f_{\Conv} \colon \cX^* \times \R^d \rightarrow \R$ defined by $f_{\Conv}(\cS,\px) \eqdef \max\set{y \in \R \colon \px \in \cC_{\cS}(y)}$.
\end{definition}

Namely, $f_{\Conv}(\cS,\px) = y$ if and only if $y$ is the maximal value such that $\px$ is a convex combination of points $\pz$ with $f(\cS,\z) \geq y$. Note that by definition it is clear that $f(\cS,\px) \leq f_{\Conv}(\cS,\px)$ for any $(\cS,\px) \in \cX^* \times \R^d$. Yet, observe that 
$\max_{\px} f_{\Conv}(\cS,\px) = M_{\cS}$.

In the following, assume that points with high value of $f_{\Conv}$ also have somewhat high value of $f$. Formally, assume 
there exists $\Delta \geq 1$ that satisfies the following requirement:
\begin{requirement}\label{eq:Delta}
	$\forall (\cS,\px) \in \cX^* \times \R^d:\text{ }f(\cS,\px) \geq \Delta \cdot f_{\Conv}(\cS,\px) - (\Delta-1) \cdot M_{\cS}$
\end{requirement}
Requirement~\ref{eq:Delta} can be interpreted as follows: For any $(\cS,\px) \in \cX^* \times \R^d$, if $f_{\Conv}(\cS,\px) = (1-\lambda)M_{\cS}$, then $f(\cS,\px) \geq (1-\lambda \Delta)M_{\cS}$. This reduces the task of finding a point $\px^*$ with $f(\cS,\px^*) \geq (1-\alpha)M_{\cS}$ to the task of finding a point $\px^*$ with $f_{\Conv}(\cS,\px^*) \geq (1 - \alpha/\Delta)M_{\cS}$.

Following the above assumption, the idea of our algorithm is to find a point $\pt{x^*}=(x_1^*,\dots,x_d^*)$ with large $f_{\Conv}$ coordinate after coordinate: we use $\AlgRecConcave$ to find a value $x^*_1$ that can be extended by some $\tx_2,\ldots,\tx_d$ so that $f_{\Conv}(x_1^*,\tx_2\ldots,\tx_d)$ is close to $M_{\cS}$, then we find a value $x_2^*$ so that there is a point $(x^*_1,x_2^*,\tx_3\dots,\tx_d)$ whose $f_{\Conv}$ is close to $M_{\cS}$, and so forth until we find all coordinates. The parameters in $\AlgRecConcave$ are set such that in each step we lose at most $\alpha M_{\cS}/\paren{d \Delta}$ from the value of $f_{\Conv}$, resulting in a point  $(x^*_1,\dots,x^*_d)$ whose $f_{\Conv}$ is at least $(1 - \alpha/\Delta)M_{\cS}$.

\subsection{Defining a Quasi-Concave Function with Small Sensitivity}

To apply the above approach, we need to prove that the functions considered in the algorithm $\AlgRecConcave$  are quasi-concave and have small sensitivity of the dataset $\cS$.

\begin{definition}\label{def:Q-general}
	For $1 \leq i \leq d$ and $x_1^*,\ldots,x_{i-1}^* \in \R$, define
	\begin{align*}
	Q_{x_1^*,\ldots,x_{i-1}^*}(\cS,x_i) \eqdef \max_{\tx_{i+1},\ldots,\tx_d \in \R} f_{\Conv}(\cS, x_1^*,\ldots,x_{i-1}^*,x_i,\tx_{i+1},\ldots,\tx_d).
	\end{align*}
\end{definition}
\vspace{-10pt}
We first prove that the function $Q_{x^*_1,\dots,x^*_{i-1}}(\cS,\cdot)$ is quasi-concave .

\begin{claim}\label{claim:quasi-concave-under-convex}
	For every $i \in [d]$ and $x_1^*,\ldots,x_{i-1}^* \in \R$, the function $Q_{x_1^*,\ldots,x_{i-1}^*}(\cS,\cdot)$ is quasi-concave.
\end{claim}
\vspace{-10pt}
\begin{proof}
	Fix $i \in [d]$ and $x_1^*,\ldots,x_{i-1}^* \in R$, and fix values $x_i,x_i',x_i'' \in \R$ such that $x_i' \leq x_i \leq x_i''$, and let $y \eqdef \min\set{Q_{x^*_1,\dots,x^*_{i-1}}(\cS,x_i'),Q_{x^*_1,\dots,x^*_{i-1}}(\cS,x_i'')}$.  
	By definition, $\exists x_{i+1}',\dots,x_{d}', x_{i+1}'',\dots,x_{d}'' \in \R$ such that both points $\px' = (x^*_1,\dots,x^*_{i-1}, x_i', x_{i+1}',\dots,x_{d}')$ and $\px'' = (x^*_1,\dots,x^*_{i-1}, x_i'', x_{i+1}'',\dots,x_{d}'')$ belong to $\cC_{\cS}(y)$. In the following, let $p \in [0,1]$ be the value such that $x_i = p x_i' + (1-p)x_i''$, and let $\px = (x^*_1,\dots,x^*_{i-1}, x_i, x_{i+1},\dots,x_{d})$ where $x_j = p x_j' + (1-p)x_j''$ for $j \in \set{i+1,\ldots,d}$. Since $\px$ lies on the line segment between $\px'$ and $\px''$, it holds that $\px \in \cC_{\cS}(y)$ (recall that $\cC_{\cS}(y)$ is a convex set). Therefore, we conclude that $Q_{x^*_1,\dots,x^*_{i-1}}(x_i) \geq y$, as required.
\end{proof}

We next prove that $Q_{x^*_1,\dots,x^*_{i-1}}(\cdot,x_i)$  has low sensitivity.

\begin{claim}\label{claim:sen-1-under-convex}
	Assume that $f$ is a sensitivity-$k$ function. Then for all $i \in [d]$ and $x_1^*,\ldots,x_{i-1}^*\in R$, $Q_{x_1^*,\ldots,x_{i-1}^*}$ is a sensitivity-$k$ function.
\end{claim}
\vspace{-10pt}
\begin{proof}
	Fix two neighboring datasets $\cS,\cS' \in \cX^*$. By assumption, it holds that $f(\cS,\px) \geq f(\cS',\px)-k$ for every $\px \in \R^d$. This yields that $\cC_{\cS'}(y) \subseteq \cC_{\cS}(y-k)$ for every $y \in \R$. Hence, we deduce by the definition of $Q_{x^*_1,\dots,x^*_{i-1}}$ that $Q_{x^*_1,\dots,x^*_{i-1}}(\cS,x_i) \geq Q_{x^*_1,\dots,x^*_{i-1}}(\cS',x_i) - k$ for every $x_i \in \R$.
\end{proof}

%\subsubsection{Determine a Finite Domain}
In order to apply algorithm $\AlgRecConcave$, for every $1\leq i \leq d$ it is required to determine a finite domain $\tcX_i = \tcX_i(x^*_1,\ldots,x^*_{i-1})$ which contains a value $x_i^*$ that reaches the maximum of $Q_{x^*_1,\dots,x^*_{i-1}}(\cS,\cdot)$ under $\R$.\footnote{
	We remark that this step might be involved for some $d$-dimensional functions,  but is inherent for privately optimizing them (at least if the optimization is done coordinate by coordinate). Yet, once we determine such domains with some finite bound $\tX$ on their sizes, it usually not blows up the resulting sample complexity of our algorithm since it only depends on $2^{O(\log^* \tX)}$ (see \cref{thm:HighDimFunc}).} Namely, we need to determined an iterative sequence of domains $\set{\tcX_i(\cdot)}_{i=1}^d$ that satisfies the following requirement:
\begin{requirement}\label{eq:domain}
	For every $\cS \in \cX^*$ and every $x^*_1,\dots,x^*_{i-1} \in \R$, it holds that
	\begin{align*}
	\exists x_i \in \tcX_i:\text{  }Q_{x^*_1,\dots,x^*_{i-1}}(\cS,x_i) = \max_{\tx_i \in \R} Q_{x^*_1,\dots,x^*_{i-1}}(\cS,\tx_i).
	\end{align*}
\end{requirement}

\subsection{The Algorithm}
In \cref{fig:OptimizeHighDimFunc}, we present an $(\eps,\delta)$-differentially private algorithm $\AlgOptimizeHighDimFunc$ that finds with probability at least $1-\beta$  a point $\px^* \in \R^d$ with $f(\cS,\px^*) \geq (1-\alpha) M_{\cS}$.

\begin{figure}[thb!]
	\begin{center}
		\noindent\fbox{
			\parbox{.95\columnwidth}{
				\begin{center}{ \bf Algorithm $\AlgOptimizeHighDimFunc$}\end{center}
				\remove{
					{\bf Preprocessing:}
					\begin{itemize}
						\item Construct the sets $\tilde{X}_1,\ldots,\tilde{X}_d$ as  in Claim~\ref{cl:sets}. Let $\tilde{T}=\max_{1\leq i \leq d} |\tilde{X_i}|$. 
						\\$(*$ By Claim~\ref{cl:sets}, $\log^* \tilde{T}=\log^* d +\log^*T +O(1)$. $*)$
					\end{itemize}
					
					{\bf Algorithm:} 
				}
				\begin{itemize}[topsep=-3pt, rightmargin=5pt]%[topsep=-3pt,itemsep=-3pt]
					\item [(i)] 
					Let $\alpha,\beta,\eps,\delta \in (0,1)$ be the utility/privacy parameters, let $\cS \in \cX^*$ be an input dataset, let $\set{\tcX_i(\cdot)}_{i=1}^d$ be an iterative sequence of finite domains, and let $\Delta \geq 1$.
					\item [(ii)] For $i=1$ to $d$ do:
					\begin{itemize}
						\item [(a)]
						Let $Q_{x^*_1,\dots,x^*_{i-1}}$ be the function from Definition~\ref{def:Q-general}.
						\item [(b)] Let $\tcX_{i} = \tcX_{i}(x^*_1,\dots,x^*_{i-1})$.
						\item [(c)]
						Execute $\AlgRecConcave$ with the function $Q_{x^*_1,\dots,x^*_{i-1}}$, domain $\tcX_{i}$, and parameters:	
							
						$r=(1-\frac{\alpha}{2d\Delta})^{i-1} M_{\cS}$, $\tilde{\alpha}=\frac{\alpha}{2d\Delta},\tilde{\beta}=\frac{\beta}{d},\tilde{\eps}=\frac{\eps}{2\sqrt{2d\ln(2/\delta)}},\tilde{\delta}=\frac{\delta}{2d}$.
						
						Let $x^*_i$ be its output.
						
					\end{itemize}
					\item [(iii)] Return $\px^* = (x^*_1,\dots,x^*_d)$.\\
				\end{itemize}
		}}
	\end{center}
	\caption{Algorithm for finding a point $\px^* \in \R^d$ with $f(\cS,\px^*) \geq (1-\alpha) M_{\cS}$.\label{fig:OptimizeHighDimFunc}}
\end{figure}

The following theorem summarizes the properties of $\AlgOptimizeHighDimFunc$.

\begin{theorem}\label{thm:HighDimFunc}
	Let $\cX$ be a domain and $f \colon \cX^* \times \R^d \rightarrow \R$ be a sensitivity-1 function. Let $\Delta\geq 1$ be a value that satisfies Requirement~\ref{eq:Delta}, and let $\set{\tcX_{i}(\cdot)}_{i=1}^d$ be an iterative sequence of finite domains that satisfies Requirement~\ref{eq:domain} (all with respect to $f$). In addition, let $\alpha,\beta,\eps \leq 1$, $\delta < 1/2$, and let $\cS \in \cX^*$ be a dataset with $M_{\cS} \eqdef \max_{\px \in \R^d} f(\cS,\px) \geq \Omega\Biggl(\Delta\cdot d^{1.5}\cdot 2^{O(\log^*\tX)}\cdot \frac{\log^{1.5}\bigl(\frac{1}{\delta}\bigr) \log \bigl(\frac{d}{\beta}\bigr)}{\eps \alpha}\Biggr)$, 
	where $\tX \eqdef \max_{i,x_1^*,\ldots,x_{i-1}^*} \size{\tcX_{i}(x_1^*,\ldots,x_{i-1}^*)}$.
	Then, $\AlgOptimizeHighDimFunc$ is an $(\eps,\delta)$-differentially private algorithm
	that with probability $1-\beta$ returns a point $\px^* \in \R^d$ with
	$f(\cS,\px^*) \geq (1-\alpha) M_{\cS}$.
\end{theorem}
\begin{proof}
	By Claims \ref{claim:quasi-concave-under-convex} and \ref{claim:sen-1-under-convex}, the proof follows similarly to Theorem 20 of \cite{BeimelMNS19} using the properties of $\AlgRecConcave$. For completeness, we give the full details below.
	\paragraph{Utility.}
	We prove by induction that after step $i$ of the algorithm, with probability at least $1-i\beta/d$, 
	the returned values $x^*_1,\ldots,x^*_i$ satisfy
	$Q_{x^*_1,\dots,x^*_{i-1}}(\cS,x^*_i)\geq (1-\frac{\alpha}{2d\Delta})^{i} M_{\cS}$,  i.e., there are $x_{i+1},\ldots,x_d \in \R$ 
	such that $f_{\Conv}(\cS,x^*_1,\dots,x^*_{i},x_{i+1},\dots,x_d)\geq (1-\frac{\alpha}{2d\Delta})^{i} \size{\cS}$. This concludes the utility part since after the $d$ iterations, with probability $1-\beta$, $\AlgOptimizeHighDimFunc$ outputs a point $\px^*$ with $f_{\Conv}(\cS,\px^*) \geq (1-\frac{\alpha}{2d\Delta})^{d} M_{\cS}\geq (1-\frac{\alpha}{\Delta}) M_{\cS}$ (follows by the inequality $1- x/2 \geq e^{-x}$ for $x \in [0,1]$), and by assumption on $\Delta$ we deduce that $f(\cS,\px^*) \geq (1-\alpha) M_{\cS}$.
	
	The basis of the induction is $i=1$: By the assumption on $\set{\tcX_{i}(\cdot)}_{i=1}^d$, there exists a value in $\tcX_1$ that maximize $Q(\cS,\cdot)$. By Proposition~\ref{prop:aRecConcave} along with the assumption on $M_{\cS}$, $\AlgRecConcave$ finds with probability at least $1 - \beta/d$ a point $x_1^*  \in \tcX_1$ with $Q(\cS,x_1^*) \geq (1-\frac{\alpha}{2d\Delta})M_{\cS}$.
	
	Next, by the induction hypothesis for $i-1$, it holds that $\max_{x_i \in \R}\set{Q_{x^*_1,\dots,x^*_{i-1}}(\cS,x_i)} \geq (1-\frac{\alpha}{2d\Delta})^{i-1} M_{\cS}$ with probability at least $1-(i-1)\beta/d$, and recall that by assumption there exists $x_i \in \tcX_i = \tcX_i(x^*_1,\dots,x^*_{i-1})$ that reaches the maximum of $Q_{x^*_1,\dots,x^*_{i-1}}(\cS,\cdot)$.
	Therefore, by Proposition~\ref{prop:aRecConcave} along with the assumption on $M_{\cS}$, with probability at least 
	$(1-\beta/d)\bigl(1-(i-1)\beta/d\bigr)\geq 1 - i\beta/d$, Algorithm
	$\AlgRecConcave$ returns $x^*_i\in \tcX_i$ with $Q_{x^*_1,\dots,x^*_{i-1}}(\cS,x_i^*) \geq (1-\frac{\alpha}{2d\Delta})^i M_{\cS}$.
	
	\paragraph{Privacy.}
	By Proposition~\ref{prop:aRecConcave} and Claim~$\ref{claim:sen-1-under-convex}$, each invocation of $\AlgOptimizeHighDimFunc$ is
	$(\tilde{\eps},\tilde{\delta})$-differentially private. $\AlgOptimizeHighDimFunc$ 
	invokes $\AlgRecConcave$ $d$ times. 
	Thus, by \cref{thm:composition2} (the advanced composition) with $\delta'=\delta/2$,
	it follows that $\AlgOptimizeHighDimFunc$ is $(\frac{\eps}{2}+\frac{\eps^2}{4\ln (2/\delta)},\delta)$ differentially-private,
	which implies $(\eps,\delta)$-privacy whenever $\eps \leq 1$ and~$\delta \leq 1/2$.
\end{proof}

%% file: LearningDeepPoints.tex
\section{The Linear Feasibility Problem}\label{sec:deepPoints}
In this section we show how the method from \cref{sec:OptHighDimFunc} can be used for privately approximating the linear feasibility problem.
In this problem, we are given a finite grid $\cX = \iseg{\pm X} \eqdef \set{x \in \Z \colon \size{x} \leq X}$ for some $X \in \N$ and a dataset  $\cS \in (\cX^d \times \cX)^*$ such that each $(\pa,w)\in \cS$ represents the linear constraint $\iprod{\pa,\px} \geq w$ which defines the halfspace $\hs_{\pa,w}$ in $\R^d$. In the following, we let $\depth_{\cS}(\px) \eqdef \size{\set{(\pa,w) \in \cS \colon \px \in \hs_{\pa,w}}}$ (that is, the number of halfspaces in $\cS$ that contain the point $\px$).
Our goal is to describe, given $\alpha,\beta,\eps,\delta \in (0,1)$, an $(\eps,\delta)$-differential private algorithm that satisfies the following utility guarantee: Given a realizable dataset of halfspaces (i.e., there exists a point $\px \in \R^d$ with $\depth_{\cS}(\px) = \size{\cS}$), then with probability $1-\beta$ the algorithm should output a point $\px^*$ with $\depth_{\cS}(\px^*) \geq (1-\alpha)\size{\cS}$.

In the following, let $\cdepth$ be the convexification of the function $\depth$ (according to Definition~\ref{def:convexification}). That is, $\cdepth_{\cS}(\px) = f_{\Conv}(\cS,\px)$ for the function $f(\cS,\px) = \depth_{\cS}(\px)$.
As a first step towards applying \cref{thm:HighDimFunc} for maximizing $\depth$, we need to determine a value $\Delta \geq 1$ that satisfies Requirement~\ref{eq:Delta}.
Namely, we need to lower bound $ \depth_{\cS}(\px)$ in terms of $\cdepth_{\cS}(\px)$ and $M_{\cS} = \size{\cS}$. For that, we prove the following claim.

\begin{claim}\label{claim:cdepth-to-depth}
	For any $\cS \in (\R^d \times \R)^*$ and any $\px \in \R^d$, it holds that
	\begin{align*}
	\depth_{\cS}(\px) \geq (d+1) \cdot \cdepth_{\cS}(\px) - d\size{\cS}.
	\end{align*}
\end{claim}
\begin{proof}
	Fix $\cS \in (\R^d \times \R)^*$ and $\px \in \R^d$, and let $k = \cdepth(\px)$. By definition it holds that $\px \in \chull(\cD_{\cS}(k))$ for $\cD_{\cS}(k) = \set{\px' \colon \depth_{\cS}(\px') \geq k}$. Therefore, by Caratheodory's theorem (Fact~\ref{fact:Caratheodory}) it holds that $\px$ is a convex combination of at most $d+1$ points $\px_1,\ldots,\px_{d+1} \in \cD_{\cS}(k)$. In the following, for $\px' \in \R^d$ let $\cT_{\px'} \eqdef \set{(\pa,w) \in \cS \colon \px' \notin \hs_{\pa,w}}$ and observe that $\depth_{\cS}(\px') = \size{\cS} - \size{\cT_{\px'}}$. Therefore, because for all $i \in [d+1]$ we have $\depth(\px_i) \geq k$, it holds that $\size{\cT_{\px_i}} \leq \size{\cS} - k$. Furthermore, note that $\cT_{\px} \subseteq \bigcup_{i=1}^d \cT_{\px_i}$ (holds since each halfspace that contains a set of points also contains any convex combination of them). We conclude that $\depth_{\cS}(\px) \geq \size{\cS} - \sum_{i=1}^{d+1} \size{\cT_{\px_i}} \geq \size{\cS}-(d+1)(\size{\cS} - k) = (d+1)k - d\size{\cS}$.
\end{proof}
Namely, $\Delta = d+1$ satisfies Requirement~\ref{eq:Delta} for the function $f(\cS,\px) = \depth_{\cS}(\px)$.

\noindent The second step towards applying \cref{thm:HighDimFunc} is to determine an iterative sequence of finite domains $\set{\tcX_i(\cdot)}_{i=1}^d$ that satisfies Requirement~\ref{eq:domain}. Namely, our goal is to determine a finite domain $\tcX_i = \tcX_i(x^*_1,\ldots,x^*_{i-1})$ such that there exists $x_i^* \in \tcX_i$ that reaches the maximum of $Q_{x^*_1,\dots,x^*_{i-1}}(\cS,\cdot)$ under $\R$, where $Q_{x^*_1,\dots,x^*_{i-1}}$ is defined below.

\begin{definition}\label{def:Q_i}
	For every $1 \leq i \leq d$ and every $x^*_1,\dots,x^*_{i-1} \in \R$, define
	$$Q_{x^*_1,\dots,x^*_{i-1}}(\cS,x_i)\eqdef \max_{\tx_{i+1},\dots,\tx_d \in \R} \cdepth_{\cS}(x^*_1,\dots,x^*_{i-1}, x_i, \tx_{i+1},\dots,\tx_d).$$
	%\triangleq
\end{definition}

The following lemma, proven in \cref{sec:missingProofs:max-in-inters-lemma}, states that at least one of the maximum points $x_i^*$ can be derived by solving a system of linear equations with bounded coefficients.

\def\maxInIntersecLemma{
	Let $X \in \N$, $\cX = \iseg{\pm X}$, $\cS \in (\cX^d \times \cX)^*$, $i \in [d]$, let $x_1^*,\ldots, x_{i-1}^* \in \R$ and let $Q_{x^*_1,\dots,x^*_{i-1}}$ be the function from Definition~\ref{def:Q_i}.  Then there exists an invertible matrix $\pA \in \cX^{(d-i+1)\times(d-i+1)}$ and values $$b_i,\ldots,b_{d} \in \cX - \sum_{j=1}^{i-1} x^*_j\cdot \cX \eqdef \bigcup_{w,a_1,\ldots,a_{i-1} \in \cX}\set{w - \sum_{j=1}^{i-1} a_j x_j^*}$$ such that $(x_i^*,\ldots,x_d^*)^T \eqdef \pA^{-1}\cdot (b_i,\ldots,b_{d})^T$  satisfies $$\cdepth_{\cS}(x_1^*,\ldots,x_d^*) = Q_{x^*_1,\ldots,x^*_{i-1}}(\cS,x_i^*) = \max_{x_i \in \R}\set{Q_{x^*_1,\ldots,x^*_{i-1}}(\cS,x_i)}.$$
}

\begin{lemma}\label{lem:max-in-intersection}
	\maxInIntersecLemma
\end{lemma}

Using Lemma~\ref{lem:max-in-intersection}, we can now define a finite domain for each iteration $i\in [d]$.

\begin{definition}[The domain $\tcX_i = \tcX_i(x^*_1,\ldots,x^*_{i-1})$]\label{def:domain-tcXi}
	We define the domains $\set{\tcX_i}_{i=1}^d$ iteratively. For $i=1$ let $\tcX_1 \eqdef \tcX_1' / \tcX_1''$ where $\tcX_1' \eqdef \iseg{\pm (d\cdot d!) \cdot X^{d}}$ and $\tcX_1'' \eqdef \paren{\iseg{\pm d!\cdot X^{d}}}\setminus \set{0}$.\footnote{Recall that for $a \in \Z^+$ we let $\iseg{\pm a} = \set{-a,-a+1,\ldots,a}$.} For $i > 1$ and given $x^*_j = s_j/t_j \in \tcX_j' / \tcX_j'' = \tcX_j$ for $j \in [i-1]$, define $\tcX_i = \tcX_i(x^*_1,\ldots,x^*_{i-1}) \eqdef \tcX_i' / \tcX_i''$ where $\tcX_i' \eqdef \iseg{\pm (d\cdot d!)^i \cdot X^{di}}$ and $\tcX_i'' = \tcX_i''(t_{i-1}) \eqdef \paren{\iseg{\pm d!\cdot X^{d}}\cdot t_{i-1}}\setminus \set{0}$.
\end{definition}

We next prove that the above sequence $\set{\tcX_i(\cdot)}_{i=1}^d$ satisfies Requirement~\ref{eq:domain}.

\begin{lemma}\label{lemma:domain-tcXi}
	Let $X \in \N$, $\cX = \iseg{\pm X}$, $\cS \in (\cX^d \times \cX)^*$, $i \in [d]$ and $x^*_1 \in \tcX_1,\ldots,x^*_{i-1} \in \tcX_{i-1}$, where $\tcX_j = \tcX_j(x^*_1,\dots,x^*_{i-1})$, for $j \in [i]$, is according to Definition~\ref{def:domain-tcXi}. Then there exists $x_i^* \in \tcX_i$ such that 
	\begin{align}\label{eq:x_i-reach-max}
	Q_{x^*_1,\dots,x^*_{i-1}}(\cS,x_i^*) = \max_{x_i \in \R}\set{Q_{x^*_1,\dots,x^*_{i-1}}(\cS,x_i)}
	\end{align}
\end{lemma}
\begin{proof}
	We are given $x^*_1 \in \tcX_1, \ldots, x^*_{i-1} \in \tcX_{i-1}$ such for all $j \in [i-1]$: $x_j^* = s_j/ t_j$ for some $s_j \in \iseg{\pm (d\cdot d!)^j \cdot X^{dj}}$ and $t_j \in \paren{\iseg{\pm d!\cdot X^{d}} \cdot t_{j-1}}\setminus \set{0}$ (letting $t_0 = 1$), and our goal is to prove the existence of $x_i^* \in \tcX_i$ that satisfies \cref{eq:x_i-reach-max}. 
	By Lemma~\ref{lem:max-in-intersection}, there exist an invertible matrix $\pA \in \cX^{(d-i+1)\times(d-i+1)}$ and values $b_i,\ldots,b_d$ with 
	\begin{align*}
	b_j
	&\in \cX - \sum_{j=1}^{i-1} x^*_j\cdot \cX = \frac{t_{i-1}\cdot \cX - \sum_{j=1}^{i-1} s_j \cdot (t_{i-1}/t_j)\cdot \cX}{t_{i-1}}\\
	&\in \frac{\iseg{\pm d!^{i-1} \cdot X^{d(i-1)+1}} + \sum_{j=1}^{i-1} \iseg{\pm (d\cdot d!)^j\cdot X^{dj}}\cdot \iseg{\pm d!^{i-1-j} \cdot X^{d(i-1-j)}} \cdot \iseg{\pm X}}{t_{i-1}}\\
	&\in \frac{1}{t_{i-1}} \cdot \iseg{\pm d^{i}\cdot d!^{i-1} \cdot X^{d(i-1)+1}},
	\end{align*}
	such that the unique solution $(x_i^*,\ldots,x_d^*)$ to the system of linear equations $\pA (x_i,\ldots,x_d)^T = (b_i,\ldots,b_d)^T$ satisfies $Q_{x^*_1,\ldots,x^*_{i-1}}(\cS,x_i^*) = \max_{x_i \in \R}\set{Q_{x^*_1,\ldots,x^*_{i-1}}(\cS,x_i)}$. 
	Hence, we deduce by Cramer's rule that
	\begin{align*}
	x_i^* 
	&= \frac{\det(\tilde{\pA})}{\det(\pA)} \in \frac{\sum_{j=i}^d b_j\cdot \iseg{\pm (d-1)! \cdot X^{d-i}}}{\iseg{\pm d! \cdot X^{d-i+1}}\setminus \set{0}}\\
	&\subseteq \frac{1}{t_{i-1}} \cdot \iseg{\pm d^i\cdot d!^{i-1} \cdot X^{d(i-1)+1}} \cdot \frac{\iseg{\pm d! \cdot X^{d-i}}}{\iseg{\pm d! \cdot X^{d-i+1}}\setminus \set{0}} \subseteq \tcX_i,
	\end{align*}
	where  $\tilde{\pA}$ is the matrix $\pA$ when replacing its first column with $(b_i,\ldots,b_d)^T$.
\end{proof}

\subsection{The Algorithm}
In \cref{fig:DeepPoint}, we present an $(\eps,\delta)$-differentially private algorithm $\AlgFindDeepPoint$ that given a realizable dataset of halfspaces $\cS$, finds with probability at least $1-\beta$  a point whose depth is at least $(1-\alpha) \size{\cS}$.

\begin{figure}[thb!]
	\begin{center}
		\noindent\fbox{
			\parbox{.95\columnwidth}{
				\begin{center}{ \bf Algorithm $\AlgFindDeepPoint$}\end{center}
				\begin{itemize}[topsep=-3pt, rightmargin=5pt]%[topsep=-3pt,itemsep=-3pt]
					\item [(i)] Let $\alpha,\beta,\eps,\delta \in (0,1)$ be the utility/privacy parameters, and let $\cS \in \paren{\cX^d \times \cX}^*$ be an input dataset.
					\item [(ii)] Execute $\AlgOptimizeHighDimFunc$ on the function $f(\cS,\cdot) = \depth_{\cS}(\cdot)$, with parameters $\alpha,\beta,\eps,\delta$, $\Delta=d+1$ and the sequence $\set{\tcX_{i}(\cdot)}_{i=1}^d$ defined in \cref{def:domain-tcXi}.
					\item [(ii)] Output the resulting point $\px^*$.
				\end{itemize}
		}}
	\end{center}
	\caption{Algorithm $\AlgFindDeepPoint$ for finding a point $\px^* \in \R^d$ with $\depth_{\cS}(\px^*) \geq (1-\alpha) \size{\cS}$.\label{fig:DeepPoint}}
\end{figure}

\begin{theorem}[Restatement of Theorem~\ref{thm:LP}]\label{thm:deep-point}
	Let $\alpha,\beta,\eps \leq 1$, $\delta < 1/2$, $X \in \N$, $\cX = \iseg{\pm X}$ and let $\cS \in \paren{\cX^d \times \cX}^*$ be a realizable dataset of halfspaces  with
	\begin{align*}
	\size{\cS} =O\Biggl(d^{2.5}\cdot 2^{O(\log^*X +\log^*d)}\frac{\log^{1.5}\bigl(\frac{1}{\delta}\bigr) \log \bigl(\frac{d}{\beta}\bigr)}{\eps \alpha}\Biggr).
	\end{align*}
	Then, $\AlgFindDeepPoint$ is an $(\eps,\delta)$-differentially private algorithm
	that with probability at least $1-\beta$ returns a point $\px^* \in \R^d$ with
	$\depth(\px^*)\geq (1-\alpha) \size{\cS}$.  Furthermore, $\AlgFindDeepPoint$ runs in time $$T = \poly(d) \cdot \size{\cS}\cdot \paren{\size{\cS}^d\cdot \log X + \polylog(1/\alpha,1/\beta,1/\eps,1/\delta,X)}.$$
\end{theorem}
Since $\depth$ is a sensitivity-1 function, the proof of \cref{thm:deep-point} immediately follow by \cref{thm:HighDimFunc}, Claim \ref{claim:cdepth-to-depth} and Lemma~\ref{lemma:domain-tcXi}. See \cref{sec:MissingProofs:implementation-details} for the running time analysis.

%% file: LearningHalfspaces.tex
\section{Learning Halfspaces}\label{sec:halfspaces}
In this section we describe our private empirical risk minimization (ERM) learner of halfspaces, and at the end we state our (almost) immediate corollary about private PAC learning.

In the considered problem, we are given a finite grid $\cX = \iseg{\pm X} \eqdef \set{x \in \Z \colon \size{x} \leq X}$ for some $X \in \N$ and a dataset of labeled points $\cS \in (\cX^d \times \set{-1,1})^*$. We say that $\cS$ is a realizable dataset of points if there exists $(\pa,w) \in \R^d \times \R$ with $\error_{\cS}(c_{\pa,w}) \eqdef \size{\set{(\px,y) \in \cS \colon c_{\pa,w}(\px) \neq y}}/\size{\cS}= 0$, letting $c_{\pa,w}\colon \cX^d \mapsto \set{-1,1}$ be the concept function that outputs $1$ iff $\px \in \hs_{\pa,w}$. Our goal is to describe, given $\alpha,\beta,\eps,\delta \in (0,1)$, an $(\eps,\delta)$-differential private algorithm that satisfies the following utility guarantee: 
Given a realizable dataset of points $\cS$, the algorithm should output
with probability $1-\beta$  a pair $(\pa^*,w^*)$ with $\error_{\cS}(c_{\pa^*,w^*}) \leq \alpha$

\subsection{A Reduction to the Linear Feasibility Problem}

We reduce the problem of learning a halfspace to the linear feasibility problem by using geometric duality between points and halfspaces.
Formally, we translate a halfspace $\hs_{\pa,w}$ to the point $(\pa,w) \in \R^{d+1}$, and translate a labeled point $(\px,y) \in \cS$ to the $(d+1)$-dimensional halfspace $\hs_{(y\cdot \px,-y), 0}$ which equals to $\set{(\pa,w) \in \R^{d+1}\colon \iprod{\pa,\px} \geq w}$ if $y=1$, and to  $\set{(\pa,w) \in \R^{d+1}\colon \iprod{\pa,\px} \leq w}$ if $y=-1$. By definition, for any realizable dataset of points $\cS$, the multiset $\cS' = \set{\paren{(y\cdot \px, -y},0) \colon (\px,y) \in \cS}$ is a realizable dataset of halfspaces. Therefore, by applying $\AlgFindDeepPoint$ on $\cS'$ we obtain a deep point $(\pa^*,w^*) \in \R^{d+1}$ for $\cS'$, meaning that $\iprod{\pa^*,y\cdot \px} \geq y\cdot w^*$ for most of the $(\px,y) \in \cS$, which is (almost) what we need.  
The problem is that the pairs $(\px,-1) \in \cS$ with $\iprod{\pa^*,\px} = w^*$ do not count as points in $\hs_{\pa^*,w^*}$  while they do count for the depth of $(\pa^*,w^*)$ in $\cS'$. Yet, assuming the points in $\cS$ are in general position (an assumption that can be eliminated),
then there can be at most $d$ such points.

\subsection{The Algorithm}
In \cref{fig:Halfspaces}, we present our algorithm $\AlgHalfSpace$ for learning halfspaces. Following the above intuition, the algorithm assumes that the points in $\cS$ are in general position.
\begin{figure}[thb!]
	\begin{center}
		\noindent\fbox{
			\parbox{.95\columnwidth}{
				\begin{center}{ \bf Algorithm $\AlgHalfSpace$}\end{center}
				\begin{itemize}[topsep=-3pt, rightmargin=5pt]%[topsep=-3pt,itemsep=-3pt]
					\item [(i)] Let $\alpha,\beta,\eps,\delta \in (0,1)$ be the utility/privacy parameters, and let $\cS \in \paren{\cX^d \times \set{-1,1}}^*$ be an input dataset.
					
					\item [(ii)] Execute $\AlgFindDeepPoint$ on the multiset $\cS' \eqdef \set{\paren{(y\cdot \px, -y},0) \colon (\px,y) \in \cS}$ and parameters $\alpha/2,\beta,\eps,\delta$.
					
					\item [(ii)] Output the resulting point $(\pa^*,w^*) \in \R^{d+1}$.
				\end{itemize}
		}}
	\end{center}
	\caption{Algorithm $\AlgHalfSpace$ for learning halfspaces.\label{fig:Halfspaces}}
\end{figure}

The following theorem summarizes the properties of $\AlgHalfSpace$.

\begin{theorem}[Private ERM learner]\label{thm:Halfspaces}
	Let $\alpha,\beta,\eps \leq 1$, $\delta < 1/2$, $X \in \N$, $\cX = \iseg{\pm X}$ and let $\cS \in \paren{\cX^d \times \set{-1,1}}^*$ be a realizable dataset of points with $\size{\cS} =O\Biggl(d^{2.5}\cdot 2^{O(\log^*X +\log^*d)}\frac{\log^{1.5}\bigl(\frac{1}{\delta}\bigr) \log \bigl(\frac{d}{\beta}\bigr)}{\eps \alpha}\Biggr)$. 
	$\AlgHalfSpace$ is $(\eps,\delta)$-differentially private. Moreover, assuming that the points in $\cS$ are in general position,\footnote{A set of points in $\R^d$ are in general position if there are no $d+1$ points that lie on the same hyperplane.}
	then with probability $1-\beta$ the algorithm returns a pair $(\pa^*,w^*) \in \R^d \times \R$ with $\error_{\cS}(c_{\pa^*,w^*})\leq \alpha$. %Furthermore, $\AlgHalfSpace$ runs in time $T_{d+1}$, letting $T_{d+1}$ be the running time of $\AlgFindDeepPoint$ on a $(d+1)$-dimensional $\size{\cS}$-size dataset and parameters $\alpha,\beta,\eps,\delta,X$.
\end{theorem}
\begin{proof}
	~
	\paragraph{Utility.}
	Since $\cS$ is a realizable dataset of points, it holds that $\cS'$ is a realizable dataset of halfspaces. 
	Therefore, by Theorem~\ref{thm:deep-point}, it holds that with probability $1-\beta$, algorithm $\AlgFindDeepPoint$ finds $(\pa^*,w^*) \in \R^{d+1}$ with $\depth_{\cS'}(\pa^*,w^*) \geq (1-\alpha/2)\size{\cS}$, meaning that $\iprod{y\cdot \px,\pa^*} - y\cdot w^* \geq 0$ for $(1-\alpha/2)\size{\cS}$ of the pairs $(\px,y) \in \cS$. Since the points in $\cS$ are in general position, by the assumption on $\size{\cS}$ there are at most $d < \alpha \size{\cS}/2$ pairs $(\px,-1) \in \cS$ that satisfy $ \iprod{\pa^*,\px} = w^*$. Overall we obtain that $\error_{\cS}(c_{\pa^*,w^*}) \leq (\size{\cS}- \depth_{\cS'}(\pa^*,w^*) +  d)/\size{\cS} \leq \alpha$.
	
	\paragraph{Privacy.}
	Follows by the privacy guarantee of $\AlgFindDeepPoint$.
	
\end{proof}

We show how to remove the assumption that the points in $\cS$ are in general position. Hence, since the VC dimension of $\halfspace(\R^d)$  is only $d+1$, we immediately obtain a private PAC learner from our private ERM learner, which is the main result of this paper.

\def\ThmMain{
	Let $\alpha,\beta,\eps \leq 1$, $\delta < 1/2$, $X \in \N$ and let $\cX = \iseg{\pm X}$. Then there exists an $(\eps,\delta)$-differentially private $(\alpha,\beta)$-PAC learner with sample complexity $s$ for the class $\halfspace(\cX^d)$ for $s = O\Biggl(d^{2.5}\cdot 2^{O\left(\log^*X +\log^*d + \log^*\paren{\frac1{\alpha \beta \eps \delta}}\right)}\cdot \frac{\log^{1.5}\bigl(\frac{1}{\delta}\bigr) \log \bigl(\frac{d}{\alpha \beta}\bigr)}{\eps \alpha}\Biggr)$.
	%The learner runs in time $\tilde{O}(T_{d+1})$, letting $T_{d+1}$ be the running time of $\AlgHalfSpace$ on an $(d+1)$-dimensional $s$-size dataset $\cS$ and parameters $\alpha,\beta,\eps,\delta,X$.}
}
\begin{theorem}[Private PAC learner, restatement of Theorem~\ref{thm:introHalfspaces}]\label{thm:main}
	\ThmMain
\end{theorem}

The proof details of Theorem~\ref{thm:main} appear at \cref{sec:missingProofs:cor-halfspaces}. In the following section, we sketch the main technical challenges in the proof. 

\subsection{Proof Overview of Theorem~\ref{thm:main}}\label{sec:proof-overview}

It is well known that given large enough dataset $\cS$ of samples drawn i.i.d. from a distribution $\mu$ and labeled according to some concept function $c$, then for any hypothesis $h$, the empirical error of $h$ on $\cS$ is close to the generalization error of $h$ on the distribution $\mu$ (see for example Theorem~\ref{thm:empirical-to-PAC}). Therefore, if $\mu$ is a distribution such that $s$ independent points from it are in general position with high probability, then 
by Theorem~\ref{thm:Halfspaces} we deduce that there exists a PAC leaning algorithm with small generalization error on $\mu$. However, the above argument does not hold for arbitrary distributions since Theorem~\ref{thm:Halfspaces} promises small empirical error only when the points in the dataset are in general position. In order to overcome this difficulty, given an $s$-size dataset $\cS$,
we first add a small random noise to each of the points in $\cS$. To determine how much noise to add, we first prove in Lemma~\ref{lemma:margin} that the fact that the points are coming from a finite grid $\cX^d = \iseg{\pm X}^d$ implies that there is a margin of at least $1/(d\cdot X)^{\poly(d)}$ between the data points to a halfspace that agrees on all their labels. Moreover, in Lemma~\ref{lemma:noise} we determine the resolution of the noise that we need to take in order to guarantee general position with high probability. Now given an $s$-size dataset $\cS$ drawn from $\mu$, we just add noise (independently) to each of the points in $\cS$, where the size of the noise is smaller than the margin (to ensure that the noisy dataset remains realizable) and the resolution of the noise is high enough (to guarantee general position). This induces a (noisy) distribution $\tilde{\mu}$ that promises general position, and now we are given a realizable dataset according to it. Therefore, we obtain a PAC learning algorithm with small generalization error on $\tilde{\mu}$. We end the proof by showing that every hypothesis that is good for $\tilde{\mu}$ is also good for $\mu$.

%% file: RecConcave.tex
\section{More Details about Algorithm $\AlgRecConcave$}\label{sec:moreRecConcave}

In this section we give a more detailed explanation on how $\AlgRecConcave$ works, but we refer to \cite{BNS13b} for the full details.

Fix a sensitivity-$1$ function $Q:\cX^*\times \tcX \rightarrow\R$ and a database $\cS \in \cX^*$ such that $Q(\cS,\cdot)$ is quasi-concave, and 
assume for simplicity that $\tcX = [\tX]$. 
Given the promise that there exists $m \in \tcX$ with $Q(\cS,m) \geq r$, the task of $\AlgRecConcave$ is to find $\ell \in  \tcX$ with $Q(\cS,\ell) \geq (1-\alpha) r$. \cite{BNS13b} defined the function 
$$Q'(\cS,j) \eqdef \min\set{L(\cS,j)-(1-\alpha)r,\text{ }r - L(\cS,j+1)}$$ 
for 
$$L(\cS,j) \eqdef \max_{a,b \in \tcX,  b-a+1 = 2^j}\set{\min_{i \in \set{a, a+1,\ldots, b}} Q(\cS,i)}.$$
Then they showed that $Q'(\cS,\cdot)$ is quasi-concave and that there exists $j$ with $Q'(\cS,j) \geq r'$ for $r' = \frac{\alpha}{2} r$. Therefore, by calling $\AlgRecConcave$ recursively (now on logarithmic size domain), we obtain a number $k$ with $Q'(\cS,k) \geq (1-\alpha)r'$. Now let $P_1$ be the $2\cdot 2^{k}$ numbers before the maximum $m$, and $P_2$ be the $2\cdot 2^{k}$ numbers after $m$.
Since $L(\cS,k+1) \leq r - q(\cS,k) \leq \paren{1 - \frac{\alpha}{4}} r$,
it holds that each of $P_1$ and $P_2$ contains a point with $Q(\cS,\cdot) \leq \paren{1 - \frac{\alpha}{4}} r$. Therefore, since $Q(\cS,\cdot)$ is quasi-concave, all the numbers outside $P = P_1 \cup P_2$ have $Q(\cS,\cdot) \leq \paren{1 - \frac{\alpha}{4}} r$. The algorithm now partitions $\tcX$ into intervals of size $8 \cdot 2^{k}$ such that one of them must contain $P$.\footnote{We remark that our description of this step is slightly oversimplified. Actually, in this step the algorithm partitions $\tcX$ into intervals $\set{A_i}$ and also into intervals $\set{B_i}$ that are right-shifted by $4 \cdot 2^{k}$. Then it is promised that in one of the partitions there is an interval that contains $P$.} Then it chooses an interval using the algorithm of \cite{ThakurtaS13}, which is an instantiation of the Propose-Test-Release framework (\cite{DworkL09}), where the quality of an interval is the maximum attainable value of $Q(\cS,\cdot)$ on it. Assuming that $r$ is large enough, the mechanism will choose the interval that contains $P$ with high probability. Since $L(\cS,k) \geq q(\cS,k) + (1-\alpha)r \geq \paren{1 - \frac{3\alpha}{4}} r$, there are $2^k$ points around $m$ that all have $Q(\cS,\cdot) \geq \paren{1-\frac{3\alpha}{4}}r$. Hence, in the last step, the algorithm defines $16$ ``equally spread'' concepts inside the chosen $8\cdot 2^k$-size segment, and chooses one of them using the Exponential Mechanism (\citet{MT07}).

\subsection{Running Time}
We use the following fact about the running time of $\AlgRecConcave$.

\begin{fact}[implicit in \cite{BNS16a} (Remark 3.17)]\label{fact:run-time-recconcave}
	The running time of $\AlgRecConcave$ on the function $Q:\cX^*\times \tcX \rightarrow\R$ for $\tcX = \iseg{\pm \tX}$ and input parameters $\cS, \alpha, \beta, \eps, \delta$ is bounded by $$(T_Q + T_L)\cdot \polylog(\tX,1/\alpha,1/\beta,1/\eps,1/\delta),$$
	where $T_Q$ is the time that takes to compute $Q(\cS,i)\text{ }$ (for every $i \in \tilde{\cX}$), and $T_L$ is the time that takes to compute $L(\cS,j) \eqdef \max_{a,b \in \tcX, b-a+1 = 2^j}\set{\min_{i \in \iseg{a,b}} Q(\cS,i)}$.
\end{fact}

%% file: MissingProofs.tex
\section{Missing Proofs}\label{sec:missingProofs}

\subsection{Proving Lemma~\ref{lem:max-in-intersection}}\label{sec:missingProofs:max-in-inters-lemma}

In this section we prove Lemma~\ref{lem:max-in-intersection}.
We start by defining a decreasing point for a function $Q$.

\begin{definition}[decreasing point]\label{def:dec-point}
	Let $Q \colon \R \mapsto \R$ be a function. We say that $x^* \in \R$ is a decreasing point for $Q$ if for all $x < x^*$ it holds that $Q(x) < Q(x^*)$, or for all $x > x^*$ it holds that $Q(x) < Q(x^*)$.
\end{definition}
Note that for any dataset $\cS \in (\cX^d \times \cX)^*$ and any fixing of $x_1^*,\ldots, x_{i-1}^* \in \R$, the function $Q_{x_1^*,\ldots, x_{i-1}^*}(\cS,\cdot)$ must have a decreasing point $x_i^*$ that reaches its maximum under $\R$ (unless the function is constant). The following lemma states that each decreasing point $x_i^*$ can be determined by intersection of hyperplanes in $\cS$ under the subspace $\set{\px \in \R^d \colon (x_1,\ldots,x_{i-1}) = (x_1^*,\ldots,x_{i-1}^*)}$. Furthermore, there exists such intersection in which all the points that belongs to it have the same $\cdepth$.

\def\RunningTimeLemma{
	Let $X \in \N$, $\cX = \iseg{\pm X}$, $\cS \in (\cX^d \times \cX)^*$, $i \in [d]$, let $x_1^*,\ldots, x_{i-1}^* \in \R$, let $Q_{x^*_1,\dots,x^*_{i-1}}$ be the function from Definition~\ref{def:Q_i} and let $\tx_i \in \R$ be a decreasing point for $Q_{x^*_1,\dots,x^*_{i-1}}(\cS,\cdot)$ (according to Definition~\ref{def:dec-point}). Then there exists a subset $\cS' \subseteq \cS$ of size $\leq d-i+1$ such that the set $\cH_{\cS'} \subseteq \R^{d-i+1}$ defined by $\cH_{\cS'} \eqdef \bigcap_{(\pa,w) \in \cS'} \hp_{(a_i,\ldots,a_d), w - \sum_{j=1}^{i-1} a_j x_j^*}$ is not empty, and for every $(x_i,\ldots,x_d) \in \cH_{\cS'}$ it holds that $x_i = \tilde{x_i}$ and that $\cdepth_{\cS}(x_1^*,\ldots,x_{i-1}^*,x_i,\ldots,,x_d) = Q_{x^*_1,\ldots,x^*_{i-1}}(\cS,\tilde{x_i})$.
}
\begin{lemma}\label{lem:for-running-time}
	\RunningTimeLemma
\end{lemma}
\begin{proof}
	We start by noting that for any $(\pa,w) \in \cS$, the hyperplane $\hp_{(a_i,\ldots,a_d), w - \sum_{j=1}^{i-1} a_j x_j^*}$ and the halfspace $\hs_{(a_i,\ldots,a_d), w - \sum_{j=1}^{i-1} a_j x_j^*}$ are simply the projections of the original hyperplane $\hp_{\pa,w}$ and halfspace $\hs_{\pa,w}$ (respectively) to the subspace $\cV \eqdef \set{\px \in \R^d \colon (x_1,\ldots,x_{i-1}) = (x_1^*,\ldots,x_{i-1}^*)}$.\footnote{The projection of $\hp_{\pa,w}$ to $V$ is simply define by setting $(x_1,\ldots,x_{i-1}) = (x_1^*,\ldots,x_{i-1}^*)$ to the equation $\sum_{j=1}^d a_j x_j = w$, which yields the $(d-i+1)$-dimensional hyperplane $\sum_{j=i}^d a_j x_j = w - \sum_{j=1}^{i-1} a_j x_j^*$.} Therefore, in this projected $(d-i+1)$-subspace, for each point $(x_i,\ldots,x_d) \in \R^{d-i+1}$ we have
	\begin{align*}
	\depth_{\cS}(x^*_1,\ldots,x^*_{i-1}, x_i,\ldots,x_d)
	= \size{\set{(\pa,w) \in \cS \colon (x_i,\ldots,x_d) \in \hs_{(a_i,\dots,a_d), w - \sum_{j=1}^{i-1} a_j x_j^*}}}.
	\end{align*}
	In the following, let $\tx_i$ be a decreasing point for $Q_{x^*_1,\dots,x^*_{i-1}}(\cS,\cdot)$, let $k = Q_{x^*_1,\ldots,x^*_{i-1}}(\cS,\tx_i)$, and assume \wlg that for all $x_i < \tx_i$ it holds that $Q_{x^*_1,\ldots,x^*_{i-1}}(\cS,x_i) < k$ (the case $x_i > \tx_i$ can be handled similarly). By definition of $Q_{x^*_1,\ldots,x^*_{i-1}}(\cS,\cdot)$ and by the assumption on $\tx_i$, there exist $\tx_{i+1},\ldots,\tx_d \in \R$ such that
	\begin{align*}
	k
	&= \cdepth_{\cS}(x^*_1,\ldots,x^*_{i-1}, \tx_i,\ldots,\tx_d)\\
	&= \depth_{\cS}(x^*_1,\ldots,x^*_{i-1}, \tx_i,\ldots,\tx_d)\\ 
	&= \size{\set{(\pa,w) \in \cS \colon (\tx_i,\ldots,\tx_d) \in \hs_{(a_i,\dots,a_d), w - \sum_{j=1}^{i-1} a_j x_j^*}}}.
	\end{align*}
	For justifying the second equality, note that $\cdepth_{\cS}(x^*_1,\ldots,x^*_{i-1}, \tx_i,\ldots,\tx_d) = k$ implies by definition that $(x^*_1,\ldots,x^*_{i-1}, \tx_i,\ldots,\tx_d)$ is a convex combination of points with $\depth_{\cS} \geq k$. Since $\tx_i$ is a decreasing point, non of these points have $x_i < \tx_i$, and therefore all these points have $x_i = \tx_i$. This yields the existence of such $\tx_{i+1},\ldots,\tx_d$ with $\depth_{\cS}(x^*_1,\ldots,x^*_{i-1}, \tx_i,\ldots,\tx_d) = k$. 
	
	By the above equation, for every $x_i,\ldots,x_d \in \R$ with $x_i < \tx_i$ it holds that
	\begin{align}\label{eq:assum-on-smaller-x_i}
	\depth_{\cS}(x_1^*,\ldots,x_{i-1}^*,x_i,\ldots,x_d)
	\leq Q_{x^*_1,\ldots,x^*_{i-1}}(\cS,x_i)
	<k
	= \depth_{\cS}(x^*_1,\ldots,x^*_{i-1}, \tx_i,\ldots,\tx_d)
	\end{align}
	We now construct the set $\cS'$. We initialize it to
	\begin{align*}
	\cS' \eqdef \set{(\pa,w) \in \cS \colon (\tx_i,\ldots,\tx_d) \in \hp_{(a_i,\ldots,a_d), w - \sum_{j=1}^{i-1} a_j x_j^*}}
	\end{align*}
	Note that $\cH_{\cS'}$ is a hyperplane of dimension $\leq d-i$ (possibly the $1$-dimensional hyperplane which is just the single point $\set{(\tx_i,\ldots,\tx_d)}$). Assume towards a contradiction that $\cH_{\cS'}$ contains a point $(x_i',\ldots,x_d')$ with $x_i' \neq \tx_i$. Then $\cH_{\cS'}$ must be a hyperplane of dimension at least two that in particular contains the line (in $\R^{d-i+1}$) that is determined by $(\tx_i,\ldots,\tx_d)$ and $(x_i',\ldots,x_d')$. In particular, this line contains a point $(x_i'',\ldots,x_d'')$ with $x_i'' < \tx_i$ such that the distance between $(\tx_i,\ldots,\tx_d)$ and $(x_i'',\ldots,x_d'')$ is $\gamma/2$, letting $\gamma > 0$ be a positive bound on the distance between $(\tx_i,\ldots,\tx_d)$ to all the hyperplanes $\hp_{(a_i,\ldots,a_d),w - \sum_{j=1}^{i-1} a_j x_j^*}$ for $(\pa,w) \in \cS \setminus \cS'$ (i.e., hyperplanes that $(\tx_i,\ldots,\tx_d)$ does not lie on them).
	It is easy to verify that $(x_i'',\ldots,x_d'')$ belongs to exactly the same (projected) halfspaces that $(\tx_i,\ldots,\tx_d)$ do: Both belong to the intersection of hyperplanes that are defined by $\cS'$ (i.e., $\cH_{\cS'}$), and belong to the same side of the hyperplanes defined by $\cS\setminus \cS'$. Therefore $\depth_{\cS}(x_1^*,\ldots,x_{i-1}^*,x_i'',\ldots,x_d'') = \depth_{\cS}(x_1^*,\ldots,x_{i-1}^*,\tx_i,\ldots,\tx_d)$, in contradiction to \cref{eq:assum-on-smaller-x_i}.
	
	At this point, we constructed $\cS'$ which is not empty, and all the points in
	%$\cV_{\cS'} \eqdef \set{\px \in \cV \colon (x_i,\ldots,x_d) \in \cH_{\cS'}}$
	$\cH_{\cS'}$
	have $x_i = \tx_i$, and there is at least one point $(\tx_i, x_{i+1}, \ldots, x_d) \in \cH_{\cS'}$ with $\cdepth_{\cS}(x_1^*,\ldots,x_{i-1}^*, \tx_i, x_{i+1}, \ldots, x_d) = k$. If all the points in $\cH_{\cS'}$ reaches $k$, then we are done. Otherwise, define
	\begin{align*}
	Q'_{x_1^*,\ldots,x_{i-1}^*,\tx_i}(\cS,x_{i+1}) \eqdef \max_{\substack{(x_{i+2},\ldots,x_d) \colon\\ (\tx_i, x_{i+1}, x_{i+2} \ldots, x_d) \in \cH_{\cS'}}} \cdepth_{\cS}(x_1^*,\ldots,x_{i-1}^*,\tx_i, x_{i+1},\ldots,x_d)
	\end{align*}
	If this function is constant, then fix an arbitrary $\tx_{i+1} \in \R$ for the next iteration. Otherwise, this function must have a decreasing point $\tx_{i+1}$ with $Q'_{x_1^*,\ldots,x_{i-1}^*,\tx_i}(\cS,\tx_{i+1}) = k$. By the same arguments done before, we can add more pairs to $\cS'$ such that now all the points in $\cH_{\cS'}$ have also $x_{i+1} = \tx_{i+1}$ and still there is at least one point that reaches $\cdepth$ of $k$. In both cases, for the next iteration, one can consider now the function
	\begin{align*}
	Q'_{x_1^*,\ldots,x_{i-1}^*,\tx_i,\tx_{i+1}}(\cS,x_{i+2}) \eqdef \max_{\substack{(x_{i+3},\ldots,x_d) \colon\\ (\tx_i, \tx_{i+1}, x_{i+2}, \ldots, x_d) \in \cH_{\cS'}}} \cdepth_{\cS}(x_1^*,\ldots,x_{i-1}^*,\tx_i, \tx_{i+1}, x_{i+2},\ldots,x_d),
	\end{align*}
	for determine a value $\tilde{x}_{i+2}$, and so forth. 
	Eventually, this process must end after at most $d-i+1$ iteration, in which the resulting $\cH_{\cS'}$ satisfies that all the points that belongs to it reaches $\cdepth$ of $k$, as required.
	
	At the end of the process, in case there are more than $d-i+1$ hyperplanes in $\cS'$, then there exists at least one hyperplane which is linearly depended in the others (i.e., its coefficients vector is linearly dependent in the coefficients vectors of the other hyperplanes in $\cS'$). Therefore, by removing it from $\cS'$ it does not change the intersection $\cH_{\cS'}$. Therefore, it is possible to remove hyperplanes from $\cS'$ until $\size{\cS'} = d-i+1$. 
\end{proof}

We now ready for proving Lemma~\ref{lem:max-in-intersection}, restated below.

\begin{lemma}[Restatement of Lemma~\ref{lem:max-in-intersection}]
	\maxInIntersecLemma
\end{lemma}
\begin{proof}
	Let $k = \max_{x_i \in \R}\set{Q_{x^*_1,\ldots,x^*_{i-1}}(\cS,x_i)}$. If the function $Q_{x_1^*,\ldots,x_{i-1}^*}(\cS,\cdot)$ is constant, then the proof trivially follows. Otherwise, since the set $\cC_{\cS}(k)$ is closed, there must exists a decreasing point $x_i^*$ for $Q_{x_1^*,\ldots,x_{i-1}^*}(\cS,\cdot)$ with $Q_{x_1^*,\ldots,x_{i-1}^*}(\cS,x_i^*) = k$. By Lemma~\ref{lem:for-running-time}, there exists a subset $\cS' \subseteq \cS$ such that the set $\cH_{\cS'} \subseteq \R^{d-i+1}$ defined by $\cH_{\cS'} \eqdef \bigcap_{(\pa,w) \in \cS'} \hp_{(a_i,\ldots,a_d), w - \sum_{j=1}^{i-1} a_j x_j^*}$ is not empty, and for all $(x_i,\ldots,x_d) \in \cH_{\cS'}$ it holds that $x_i = x_i^*$ and that $\cdepth_{\cS}(x_1^*,\ldots,x_{i-1}^*,x_i,\ldots,,x_d) = k$. We can assume \wlg that the vectors $\set{(a_i,\ldots,a_d) \colon (\pa,w) \in \cS'}$ are linearly independent (otherwise, one can remove pairs from $\cS'$ without changing $\cH_{\cS'}$). If $\size{\cS'} = d-i+1$ then we are done by defining the matrix $\pA$ to be the matrix with rows $\set{(a_i,\ldots,a_d) \colon (\pa,w) \in \cS'}$. Otherwise, one can add linearly independent rows from the grid (e.g., unit vectors) without changing the properties of $\cH_{\cS'}$. The proof now follows.
\end{proof}

\subsection{Implementing $\AlgFindDeepPoint$}\label{sec:MissingProofs:implementation-details}

In this section we show how  $\AlgFindDeepPoint$ (\cref{fig:DeepPoint}) can be implemented, and we bound its running time. The formal statement appears below.
\begin{lemma}
	Let $\alpha,\beta,\eps \leq 1$, $\delta < 1/2$, $X \in \N$, $\cX = \iseg{\pm X}$ and let $\cS \in \paren{\cX^d \times \cX}^*$. Then $\AlgFindDeepPoint$ on input $\alpha,\beta,\eps,\delta,\cS$ runs in time $$T = \poly(d) \cdot \size{\cS}\cdot \paren{\size{\cS}^d\cdot \log X + \polylog(1/\alpha,1/\beta,1/\eps,1/\delta,X)}.$$
\end{lemma}
\begin{proof}
	We show how to implement in time $\poly(d) \cdot \size{\cS}\cdot \paren{\size{\cS}^d\cdot \log X + \polylog(1/\alpha,1/\beta,1/\eps,1/\delta,X)}$ each iteration $i \in [d]$ of $\AlgOptimizeHighDimFunc$ (\cref{fig:OptimizeHighDimFunc}).
	At the beginning of the iteration, we first start with a preprocessing phase that takes time $\size{\cS}^{d+1} \cdot \poly(d)\cdot \log X$ in which we construct a list $L$ of size $O(\size{\cS} \cdot \log |\tcX_i|) \leq O(d^2 \log d \cdot \size{\cS}\cdot \log X)$. This list contains all pairs $(x_i^*,k) \in \tcX_i \times [\size{\cS}]$ (in sorted order according to the first value) such that $k = Q_{x^*_1,\dots,x^*_{i-1}}(\cS,x_i^*)$ and $x_i^*$ is a decreasing point for $Q_{x^*_1,\dots,x^*_{i-1}}(\cS,\cdot)$ according to Definition~\ref{def:dec-point}.
	Furthermore, the list also contain $(-\tX_i, Q_{x^*_1,\dots,x^*_{i-1}}(\cS,-\tX_i))$ and $(\tX_i, Q_{x^*_1,\dots,x^*_{i-1}}(\cS,\tX_i))$, letting $\tX_i = \max(\tcX_i)$. 
	In order to compute $Q_{x^*_1,\dots,x^*_{i-1}}(\cS,x_i)$ for some $x_i \in \tcX_i$, we search in the list two adjacent pairs $(x_i',k')$ and $(x_i'',k'')$ such that $x_i \in [x_i', x_i'']$, and then it just holds that $Q_{x^*_1,\dots,x^*_{i-1}}(\cS,x_i) = \min\set{k',k''}$ (the direction $\geq$ is clear since the function is quasi-concave. For the other direction, note that if $Q_{x^*_1,\dots,x^*_{i-1}}(\cS,x_i) > \min\set{k',k''}$, where assume \wlg that $k' \leq k$, then there must exists a decreasing point between $x_i'$ and $x_i$ since the sets $\cC_{\cS}(\cdot)$ are close, in contradiction to the assumption that $L$ contains all decreasing points). This computation can be done in time $O(\size{\cS} \cdot \tY_i)$, letting $\tY_i = \tilde{O}(d^4 \log X)$ be the number of bits that are needed for representing all the points in $\tcX_i$. Similarly, given $j \in [\tY_i]$, computing $L(\cS,j)$ can be performed by searching pairs $(x_i',k')$ and $(x_i'',k'')$  with $x_i'' - x_i' \geq 2^j$ that maximize $\min\set{k',k''}$. This can also be implement in time $O(\size{\cS} \cdot \tY_i)$. Therefore, given the list $L$, we conclude by the above analysis along with Fact~\ref{fact:run-time-recconcave} that $\AlgRecConcave$ can be implemented in time $\poly(d)\cdot \size{\cS}\cdot \polylog(1/\alpha,1/\beta,1/\eps,1/\delta,X)$.
	
	The expensive part is constructing the list $L$.
	By Lemma~\ref{lem:for-running-time}, in order to find all decreasing points with their values, it is enough to
	go over all the $O(\size{\cS}^{d-i+1})$ intersections between at most $d-i+1$ hyperplane in the set $\bigcup_{(\pa,w) \in \cS} \set{\hp_{(a_i,\ldots,a_d), w - \sum_{j=1}^{i-1} a_j x_j^*}}$ and check whether they uniquely determine that $x_i = \tilde{x_i}$ for some $\tilde{x_i} \in \R$.
	For each such $\tilde{x_i} $, find $\tilde{x_{i+1}}, \ldots, \tilde{x_{d}} \in \R$ such that $(\tilde{x_i},\ldots,\tilde{x_{d}})$ belongs to the intersection, evaluate the depth $k = \depth_{\cS}(x_1^*,\ldots,x_{i-1}^*,\tilde{x_i},\ldots,\tilde{x_d})$ and update the list: if there exists $(x_i', k'),(x_i'',k'')$ in $L$ such that $\tilde{x_i} \in [x_i',x_i'']$  and $k \leq \min\set{k',k''}$, then ignore $\tilde{x_i}$ (it is not a decreasing point). Otherwise, insert $\tilde{x_i}$ to the list and remove all points $(x_i',k')$ that we know they are not a decreasing point after this insertion. Checking whether the intersection uniquely determine $x_i$ and finding a point in it, can be done in time $\poly(d) \log \tX_i$ using Guassian elimination. Since the size of the list is $O(\size{\cS})$ in each step, updating the list each time can be done in time $O(\size{\cS} \log \tX_i)$.
\end{proof}

\subsection{Proving Theorem~\ref{thm:main}}\label{sec:missingProofs:cor-halfspaces}

In this section we present the %missing details in the 
proof of Theorem~\ref{thm:main}.
We start by stating two lemmatas. The first lemma states that if the points in the dataset are coming from a grid $\cX^d = \iseg{\pm X}^d$, then there is a margin of $1/(d\cdot X)^{\poly(d)}$.

\def\MarginLemma{
	Let $X \in \N$, $\cX = \iseg{\pm X}$ and let $\cS \in \paren{\cX^d \times \set{-1,1}}^*$ be a realizable dataset of points. Then there exists a halfspace $\hs \subset \R^d$ with $\val_{\cS}(\hs) = \size{\cS}$ such that for all $(\px,\cdot) \in \cS$ it holds that $\dist(\px,\hs) \eqdef \min_{\px' \in \hs}\set{\norm{\px-\px'}} \geq 1/X'$, for $X' \eqdef 2d^2\cdot d!^{d^3}\cdot X^{d^6}$.
}

\begin{lemma}\label{lemma:margin}
	\MarginLemma
\end{lemma}
\begin{proof}
	We prove that $\exists \pa = (a_1,\ldots,a_d) \in \R^d$ with $a_i \in \iseg{\pm d!^d\cdot X^{d^2}}/\paren{d\cdot \iseg{\pm d!^{d}\cdot X^{d^2}}\setminus \set{0}}$ and $w \in \set{-1,0,1}$ such that $\val_{\cS}(\hs_{\pa,w}) = \size{\cS}$. This yields that for any $\px \in \cX$ we have that $$\iprod{\pa,\px} \in \iseg{\pm d!^{d^2}\cdot X^{d^4}}/\paren{d\cdot \iseg{\pm d!^{d^2}\cdot X^{d^4}}\setminus \set{0}}.$$ Therefore, for every $(\px,-1)\in \cS$, since $\iprod{\pa,\px} < w$ then it must hold that $\iprod{\pa,\px} \leq w - 1/\paren{d\cdot d!^{d^2}\cdot X^{d^4}}$. This yields that for every $(\px,-1)\in \cS$ and every $\pt{v} \in \R^d$ with $\norm{\pt{v}} < 2/X'$ it holds that
	\begin{align*}
		\iprod{\pa,\px+\pt{v}}
		\leq \iprod{\pa,\px} + \norm{\pa}\cdot \norm{\pt{v}}
		\leq \paren{w - 1/\paren{d\cdot d!^{d^2}\cdot X^{d^4}}} + \paren{d \cdot d!^d \cdot X^{d^2}} \cdot 2/X' < w.
	\end{align*}
	At this point we proved the existence of a halfspace $\hs$ with $\val_{\cS}(\hs) = \size{\cS}$ such that it is far by at least $2/X'$ from all the point $\px$ with $(\px,-1) \in \cS$. This in particular yields the existence of an halfspace $\hs'$ with $\val_{\cS}(\hs') = \size{\cS}$ that is far by at least $1/X'$ from all the points in $\cS$. 
	
	It remains to prove the existence of such $\pa$ and $w$. As explained in \cref{sec:halfspaces}, the assumption that $\cS$ is a realizable dataset of points implies that there exists $w \in \set{-1,0,1}$ such that there exists a solution $\pa = (a_1,\ldots,a_d) \in \R^d$ to the system of equations 
	$$\cE \eqdef \set{\iprod{\px,\pa} \geq w}_{(\px,1) \in \cS} \bigcup \set{\iprod{-\px,\pa} > -w}_{(\px,-1) \in \cS}.$$
	Let $\cF$ be the feasible area of $\cE$, and let $C(\cF)$ be the closure of $\cF$ which is a polytope in $\R^d$ (might be unbounded). Each vertex of $C(\cF)$ is a solution to $d$ linearly independent equations in $\set{\iprod{y\cdot \px,\pa} = y\cdot w}_{(\px,y) \in \cS}$. Therefore, for any vertex $\pa^* = (a_1^*,\ldots,a_d^*)$, it holds by Cramer's rule that
	$a^*_i \in \iseg{\pm d!\cdot X^{d-1}}/\paren{\iseg{\pm d!\cdot X^d}\setminus \set{0}}$. Let $d' \leq d$ be the (largest) value in which $C(\cF)$ has $d'$-dimensional non-zero volume. If $C(\cF)$ has less than $d'+1$ vertices, then $C(\cF)$ is unbounded and the statement trivially follows. Otherwise, the average of $d'+1$ vertices of $C(\cF)$ must be a point in $\cF$ and the proof follows since each coordinate of the average belongs to $$\sum_{j=1}^{d'} \iseg{\pm d!\cdot X^{d-1}}/\paren{d\cdot \iseg{\pm d!\cdot X^d}\setminus \set{0}} \subseteq \iseg{\pm d!^d\cdot X^{d^2}}/\paren{d\cdot \iseg{\pm d!^{d}\cdot X^{d^2}}\setminus \set{0}}$$
\end{proof}

The second lemma determines the resolution of the noise that we need to add to each of the points in $\cS$ in order to guarantee general position with high probability.

\def\NoiseLemma{
	Let $\cS \subseteq (\R^d)^*$ be a multiset, let $\beta > 0$, and let $U_{\cA}$ be the uniform distribution over a set $\cA \subset \R$ of size $\geq d \size{\cS}^d/\beta$. Let $\tilde{\cS}$ be the multiset that is generated by the following process: For each $\px = (x_1,\ldots,x_d) \in \cS$, sample $\pz = (z_1,\ldots,z_d) \sim \paren{U_{\cA}}^{d}$ (i.e., each $z_i$ is sampled independently from $U_{\cA}$), and insert $(x_1+ z_1,\ldots,x_d + z_d)$ to $\tilde{\cS}$. Then with probability at least $1-\beta$ it holds that the points in $\tilde{\cS}$ are in general position.
}

\begin{lemma}\label{lemma:noise}
	\NoiseLemma
\end{lemma}
\begin{proof}
	Note that a set of points $\cS \subset \R^d$ are in general position if for any $d+1$ points $\tilde{\px}_1=(\tilde{x}_{1,1},\ldots,\tilde{x}_{1,d}),\ldots,\tilde{\px}_{d+1} = (\tilde{x}_{d+1,1},\ldots,\tilde{x}_{d+1,d}) \in \tilde{\cS}$ it holds that the vectors $(\tilde{\px}_1 - \tilde{\px}_{d+1}), \ldots, (\tilde{\px}_d - \tilde{\px}_{d+1})$ are linearly independent, meaning that $\det\paren{(\tilde{x}_{i,j} - \tilde{x}_{d+1,j})_{i,j \in [d]}} \neq 0$.
	In the following, for $k \in [d]$, let $E_{k}$ be the event that for all $k$ points $\tilde{\px}_1,\ldots,\tilde{\px}_{k-1}, \tilde{\px}_{d+1} \in \tilde{\cS}$ it holds that the $k \times k$ matrix $(\tilde{x}_{i,j} - \tilde{x}_{d+1,j})_{i,j \in [k]}$ has determinant $\neq 0$. Our goal is to show that $\pr{E_d} \geq 1-\beta$, which yields that the points in $\tilde{\cS}$ are in general position w.p. $\geq 1-\beta$.
	We start with the event $E_1$. The event means that all the points in $\tilde{\cS}$ has first coordinate $\neq 0$. Since the first coordinate is taken uniformly from a set of size $\size{\cA}$, then by union bound the probability that one of the points has first coordinate $0$ is bounded by $\size{\cS}/\size{\cA}$, meaning that $\pr{\neg E_1} \leq \size{\cS}/\size{\cA}$. We now prove that for each $k \in [d]$ it holds that $\pr{\neg E_k \mid E_1 \land \ldots \land E_{k-1}} \leq \size{\cS}^k/\size{\cA}$. Fix $k$ points $\tilde{\px}_1,\ldots,\tilde{\px}_{k-1}, \tilde{\px}_{d+1} \in \tilde{\cS}$. Note that by computing the determinant of $(\tilde{x}_{i,j} - \tilde{x}_{d+1,j})_{i,j \in [k]}$ using its last row we get that $\det((\tilde{x}_{i,j} - \tilde{x}_{d+1,j})_{i,j \in [k]}) = (-1)^k \cdot \det\paren{(\tilde{x}_{i,j} - \tilde{x}_{d+1,j})_{i,j \in [k-1]}}\cdot  (\tilde{x}_{k,k}- \tilde{x}_{d+1,k}) + \lambda$, where $\lambda$ is independent of $\tilde{x}_{k,k}$, and $\det\paren{(\tilde{x}_{i,j} - \tilde{x}_{d+1,j})_{i,j \in [k-1]}} \neq 0$ by the conditioning. Therefore, in order for the determinant to be $0$, it must hold that $\tilde{x}_{k,k} = \tilde{x}_{d+1,k} + (-1)^{k+1}\cdot \lambda/\det\paren{(\tilde{x}_{i,j})_{i,j \in [k-1]}}$. This holds with probability at most $1/\size{\cA}$ for any such fixing of $k$ points, and therefore we deduce by union bound that $\pr{\neg E_k \mid E_1 \land \ldots \land E_{k-1}} \leq \size{\cS}^k/\size{\cA}$. We conclude that
	\begin{align*}
	\pr{E_d} \geq \pr{E_1 \land \ldots \land E_d} = 1 - \sum_{k=1}^d \pr{\neg E_k \mid E_1 \land \ldots \land E_{k-1}} \geq 1 - d\cdot \size{\cS}^d/\size{\cA} \geq 1 - \beta
	\end{align*}
\end{proof}

We now ready to prove Theorem~\ref{thm:main}, restated below.

\begin{theorem}[Restatement of Theorem~\ref{thm:main}]
	\ThmMain
\end{theorem}
\begin{proof}
	Let $\mu$ be a target distribution over points in $\cX^d$. In the following, let $X'$ be the value from Lemma~\ref{lemma:margin}, let $\Delta \eqdef \lceil d \cdot s^d/(2\beta) \rceil$, let $\Delta' \eqdef 2 \Delta \cdot X' \sqrt{d}$ and let $\cA \eqdef \iseg{\pm \Delta} / \Delta'$. We now define the (noisy) distribution $\tilde{\mu} \eqdef \mu + \paren{U_{\cA}}^d$ (Namely, $\tilde{\mu}$ is the distribution induces by the outcome of $\px + \pz$ where $\px \sim \mu$ and $\pz \sim \paren{U_{\cA}}^d$, i.e., each $z_i$ is sampled independently and uniformly from $\cA$). Note that $\tilde{\mu}$ can be seen as a distribution over points in $\tilde{\cX}^d = \iseg{\pm \tX}^d$, for $\tX \eqdef \Delta'(X + \Delta)$ (one just need to strech the points from $\mu$ by a factor of $\Delta'$ in order to guarantee that they will be on an integer grid).
	
	Consider now an $s$-size dataset $\cS \in (\cX^d \times \set{-1,1})$ where the points in $\cS$ are sampled according to $\mu$ and the labels are according to a concept function $c \in \halfspace(\cX^d)$. We now construct a dataset $\cS' \in (\tcX^d \times \set{-1,1})$, where for each $(\px,y) \in \cS$ we insert $(\px + \pz,y)$ into $\cS'$, for a random noise $\pz \sim \paren{U_{\cA}}^d$. Since, by definition, it holds that $\norm{z} < 1/X'$, then by Lemma~\ref{lemma:margin} we deduce that the dataset $\cS'$ remains realizable. By Lemma~\ref{lemma:noise}, since $\size{\cA} \geq d \size{\cS}^d/(4\beta)$, it holds that the points in $\tilde{\cS}$ are in general position (except with probability $\beta/4$). Therefore, by the above arguments and by Theorem~\ref{thm:Halfspaces},  when executing $\AlgHalfSpace$ on the dataset $\cS'$ and the parameters $\alpha/20,\beta/4,\eps,\delta$, then with probability $\geq 1-\beta/2$ the resulting hypothesis $h = c_{\pa,w}$ satisfies that $h(\px) = y$ for at least $(1-\alpha/20)|\tilde{\cS}|$ of the pairs $(\px,y)\in \tilde{\cS}$, where recall that $c_{\pa,w}(\px) = 1 \iff \px \in \hs_{\pa,w}$. By Theorem~\ref{thm:empirical-to-PAC}, 
	we decude that $\ppr{h \sim \AlgHalfSpace}{\error_{\tilde{\mu}}(c,h) \leq \alpha/2} \geq 1-\beta$. We finish the proof by showing that for every $h$ it holds that $\error_{\mu}(c,h) \leq 2\cdot \error_{\tilde{\mu}}(c,h)$. For that, note that
	\begin{align*}
	\error_{\tilde{\mu}}(c,h)
	&= \ppr{\px + \pz \sim \tilde{\mu}}{c(\px + \pz) \neq h(\px + \pz)}\\
	&\geq \ppr{\px \sim \mu}{c(\px) \neq h(\px)} \cdot \ppr{\px + \pz \sim \tilde{\mu}}{c(\px + \pz) \neq h(\px + \pz) \mid c(\px) \neq h(\px)}
	\end{align*}
	Hence, it is enough to show that for every $\px \in \R^d$ such that $c(\px) \neq h(\px)$ it holds that $c(\px + \pz) \neq h(\px + \pz)$ with probability at least $1/2$. Assume \wlg that $h(\px) = 1$ and $c(\px) = -1$ (the other case can be handled similarly). The assumption $h(\px) = 1$ implies that $\iprod{\pa,\px} \geq w$ for the $\pa,w$ that $h = c_{\pa,w}$. Note that for all $\pz \in \cA^d$ it holds that at least one of $\set{\pz,-\pz}$ satisfies $\iprod{\pa,\pz} \geq 0$ which implies that $\iprod{\pa,\px + \pz} \geq w$. We deduce that at least half of the points in $\cA^d$ satisfies $h(\px + \pz) = h(\px)$. The proof now follows since $\z$ is chosen uniformly from $\cA^d$.
\end{proof}

%% file: arXiv.bbl
\begin{thebibliography}{27}
\providecommand{\natexlab}[1]{#1}
\providecommand{\url}[1]{\texttt{#1}}
\expandafter\ifx\csname urlstyle\endcsname\relax
  \providecommand{\doi}[1]{doi: #1}\else
  \providecommand{\doi}{doi: \begingroup \urlstyle{rm}\Url}\fi

\bibitem[Alon et~al.(2019)Alon, Livni, Malliaris, and Moran]{ALMM18}
N.~Alon, R.~Livni, M.~Malliaris, and S.~Moran.
\newblock Private {PAC} learning implies finite littlestone dimension.
\newblock In \emph{Proceedings of the 51st Annual {ACM} {SIGACT} Symposium on
  Theory of Computing, {STOC} 2019, Phoenix, AZ, USA, June 23-26, 2019}, pages
  852--860, 2019.

\bibitem[Bassily et~al.(2014)Bassily, Smith, and Thakurta]{BST14}
R.~Bassily, A.~Smith, and A.~Thakurta.
\newblock Private empirical risk minimization: Efficient algorithms and tight
  error bounds.
\newblock In \emph{FOCS}, pages 464--473, 2014.
\newblock URL \url{http://dx.doi.org/10.1109/FOCS.2014.56}.

\bibitem[Beimel et~al.(2013{\natexlab{a}})Beimel, Nissim, and Stemmer]{BNS13}
A.~Beimel, K.~Nissim, and U.~Stemmer.
\newblock Characterizing the sample complexity of private learners.
\newblock In \emph{ITCS}, pages 97--110. ACM, 2013{\natexlab{a}}.

\bibitem[Beimel et~al.(2013{\natexlab{b}})Beimel, Nissim, and Stemmer]{BNS13b}
A.~Beimel, K.~Nissim, and U.~Stemmer.
\newblock Private learning and sanitization: Pure vs. approximate differential
  privacy.
\newblock In \emph{APPROX-RANDOM}, volume 8096 of \emph{Lecture Notes in
  Computer Science}, pages 363--378. Springer, 2013{\natexlab{b}}.
\newblock Journal version:{\em Theory of Computing}, 12(1):1–61, 2016.

\bibitem[Beimel et~al.(2016)Beimel, Nissim, and Stemmer]{BNS16a}
A.~Beimel, K.~Nissim, and U.~Stemmer.
\newblock Private learning and sanitization: Pure vs. approximate differential
  privacy.
\newblock \emph{Theory of Computing}, 12\penalty0 (1):\penalty0 1--61, 2016.
\newblock URL \url{https://doi.org/10.4086/toc.2016.v012a001}.

\bibitem[Beimel et~al.(2019)Beimel, Moran, Nissim, and Stemmer]{BeimelMNS19}
A.~Beimel, S.~Moran, K.~Nissim, and U.~Stemmer.
\newblock Private center points and learning of halfspaces.
\newblock In \emph{Conference on Learning Theory, {COLT} 2019, 25-28 June 2019,
  Phoenix, AZ, {USA}}, pages 269--282, 2019.

\bibitem[Blum et~al.(2005)Blum, Dwork, McSherry, and Nissim]{BDMN05}
A.~Blum, C.~Dwork, F.~McSherry, and K.~Nissim.
\newblock Practical privacy: The {SuLQ} framework.
\newblock In C.~Li, editor, \emph{PODS}, pages 128--138. ACM, 2005.

\bibitem[Blumer et~al.(1989)Blumer, Ehrenfeucht, Haussler, and
  Warmuth]{BlumerEHW89}
A.~Blumer, A.~Ehrenfeucht, D.~Haussler, and M.~K. Warmuth.
\newblock Learnability and the vapnik-chervonenkis dimension.
\newblock \emph{J. {ACM}}, 36\penalty0 (4):\penalty0 929--965, 1989.

\bibitem[Bun et~al.(2015)Bun, Nissim, Stemmer, and Vadhan]{BNSV15}
M.~Bun, K.~Nissim, U.~Stemmer, and S.~P. Vadhan.
\newblock Differentially private release and learning of threshold functions.
\newblock In \emph{{FOCS}}, pages 634--649, 2015.

\bibitem[Bun et~al.(2018)Bun, Dwork, Rothblum, and Steinke]{BunDRS18}
M.~Bun, C.~Dwork, G.~N. Rothblum, and T.~Steinke.
\newblock Composable and versatile privacy via truncated cdp.
\newblock In \emph{{STOC}}, pages 74--86, 2018.

\bibitem[Bun et~al.(2020)Bun, Livni, and Moran]{bun2020equivalence}
M.~Bun, R.~Livni, and S.~Moran.
\newblock An equivalence between private classification and online prediction.
\newblock \emph{CoRR}, abs/2003.00563, 2020.
\newblock URL \url{https://arxiv.org/abs/2003.00563}.

\bibitem[Chaudhuri et~al.(2011)Chaudhuri, Monteleoni, and
  Sarwate]{ChaudhuriMS11}
K.~Chaudhuri, C.~Monteleoni, and A.~D. Sarwate.
\newblock Differentially private empirical risk minimization.
\newblock \emph{Journal of Machine Learning Research}, 12:\penalty0 1069--1109,
  2011.

\bibitem[Dunagan and Vempala(2008)]{DV08}
J.~Dunagan and S.~Vempala.
\newblock A simple polynomial-time rescaling algorithm for solving linear
  programs.
\newblock \emph{Mathematical Programming}, 114\penalty0 (1):\penalty0 101--114,
  Jul 2008.
\newblock ISSN 1436-4646.

\bibitem[Dwork and Lei(2009{\natexlab{a}})]{DworkL09}
C.~Dwork and J.~Lei.
\newblock Differential privacy and robust statistics.
\newblock In \emph{{STOC}}, pages 371--380. {ACM}, May 31--June 2
  2009{\natexlab{a}}.

\bibitem[Dwork and Lei(2009{\natexlab{b}})]{DworkLei}
C.~Dwork and J.~Lei.
\newblock Differential privacy and robust statistics.
\newblock In M.~Mitzenmacher, editor, \emph{STOC}, pages 371--380. ACM,
  2009{\natexlab{b}}.

\bibitem[Dwork et~al.(2006{\natexlab{a}})Dwork, Kenthapadi, McSherry, Mironov,
  and Naor]{DKMMN06}
C.~Dwork, K.~Kenthapadi, F.~McSherry, I.~Mironov, and M.~Naor.
\newblock Our data, ourselves: Privacy via distributed noise generation.
\newblock In S.~Vaudenay, editor, \emph{EUROCRYPT}, volume 4004 of
  \emph{Lecture Notes in Computer Science}, pages 486--503. Springer,
  2006{\natexlab{a}}.

\bibitem[Dwork et~al.(2006{\natexlab{b}})Dwork, McSherry, Nissim, and
  Smith]{DMNS06}
C.~Dwork, F.~McSherry, K.~Nissim, and A.~Smith.
\newblock Calibrating noise to sensitivity in private data analysis.
\newblock In \emph{TCC}, volume 3876 of \emph{Lecture Notes in Computer
  Science}, pages 265--284. Springer, 2006{\natexlab{b}}.

\bibitem[Dwork et~al.(2010)Dwork, Rothblum, and Vadhan]{DRV10}
C.~Dwork, G.~N. Rothblum, and S.~P. Vadhan.
\newblock Boosting and differential privacy.
\newblock In \emph{FOCS}, pages 51--60. IEEE Computer Society, 2010.

\bibitem[Feldman and Xiao(2015)]{FX14}
V.~Feldman and D.~Xiao.
\newblock Sample complexity bounds on differentially private learning via
  communication complexity.
\newblock \emph{{SIAM} J. Comput.}, 44\penalty0 (6):\penalty0 1740--1764, 2015.
\newblock URL \url{http://dx.doi.org/10.1137/140991844}.

\bibitem[Hsu et~al.(2014)Hsu, Roth, Roughgarden, and Ullman]{HsuRRU14}
J.~Hsu, A.~Roth, T.~Roughgarden, and J.~Ullman.
\newblock Privately solving linear programs.
\newblock In \emph{{ICALP}}, pages 612--624, 2014.
\newblock URL \url{https://doi.org/10.1007/978-3-662-43948-7\_51}.

\bibitem[Kaplan et~al.(2020)Kaplan, Ligett, Mansour, Naor, and
  Stemmer]{KaplanLMNS20}
H.~Kaplan, K.~Ligett, Y.~Mansour, M.~Naor, and U.~Stemmer.
\newblock Privately learning thresholds: Closing the exponential gap.
\newblock In \emph{Conference on Learning Theory, {COLT}}, volume 125, pages
  2263--2285, 2020.

\bibitem[Kasiviswanathan et~al.(2011)Kasiviswanathan, Lee, Nissim,
  Raskhodnikova, and Smith]{KLNRS11}
S.~P. Kasiviswanathan, H.~K. Lee, K.~Nissim, S.~Raskhodnikova, and A.~D. Smith.
\newblock What can we learn privately?
\newblock \emph{{SIAM} J. Comput.}, 40\penalty0 (3):\penalty0 793--826, 2011.
\newblock URL \url{https://doi.org/10.1137/090756090}.

\bibitem[Kearns(1998)]{Kearns98}
M.~J. Kearns.
\newblock Efficient noise-tolerant learning from statistical queries.
\newblock \emph{J. ACM}, 45\penalty0 (6):\penalty0 983--1006, 1998.

\bibitem[McSherry and Talwar(2007)]{MT07}
F.~McSherry and K.~Talwar.
\newblock Mechanism design via differential privacy.
\newblock In \emph{FOCS}, pages 94--103. IEEE Computer Society, 2007.

\bibitem[Nguyen et~al.(2019)Nguyen, Ullman, and Zakynthinou]{NguyenUlZa19}
H.~L. Nguyen, J.~Ullman, and L.~Zakynthinou.
\newblock Efficient private algorithms for learning halfspaces.
\newblock \emph{CoRR}, abs/1902.09009, 2019.
\newblock URL \url{http://arxiv.org/abs/1902.09009}.

\bibitem[Thakurta and Smith(2013)]{ThakurtaS13}
A.~Thakurta and A.~D. Smith.
\newblock Differentially private feature selection via stability arguments, and
  the robustness of the lasso.
\newblock In S.~Shalev{-}Shwartz and I.~Steinwart, editors, \emph{{COLT}},
  volume~30, pages 819--850, 2013.

\bibitem[Valiant(1984)]{Valiant84}
L.~G. Valiant.
\newblock A theory of the learnable.
\newblock \emph{Commun. ACM}, 27\penalty0 (11):\penalty0 1134--1142, Nov. 1984.
\newblock ISSN 0001-0782.
\newblock URL \url{http://doi.acm.org/10.1145/1968.1972}.

\end{thebibliography}
